\documentclass{article}

\PassOptionsToPackage{numbers, compress}{natbib}
\usepackage[main, final]{neurips_2026}

\usepackage[utf8]{inputenc} % allow utf-8 input
\usepackage[T1]{fontenc}    % use 8-bit T1 fonts
\usepackage{hyperref}       % hyperlinks
\usepackage{url}            % simple URL typesetting
\usepackage{booktabs}       % professional-quality tables
\usepackage{amsfonts}       % blackboard math symbols
\usepackage{nicefrac}       % compact symbols for 1/2, etc.
\usepackage{microtype}      % microtypography
\usepackage[table]{xcolor}  % colors and table support before later package loads

\usepackage[export]{adjustbox}
\usepackage[ruled]{algorithm2e}
\usepackage[inline, shortlabels]{enumitem}
\usepackage[T1]{fontenc}
\usepackage{hyperref}
\usepackage{microtype}
\usepackage{pifont}
\usepackage{xcolor}
\usepackage{xurl}
\usepackage{graphicx}
\usepackage{booktabs}
\usepackage{tabularray}
\usepackage{listings}
\usepackage{amsmath, amsfonts}
\usepackage{nicefrac}

\UseTblrLibrary{booktabs}

\makeatletter
\newcommand{\ssymbol}[1]{\@fnsymbol{#1}}
\newcommand{\romanNumeral}[1]{\expandafter\@slowromancap\romannumeral #1@}
\makeatother

\usepackage{physics}
\usepackage{mathtools}
\usepackage{amsthm}
\usepackage{amssymb}
\usepackage{subcaption}
\usepackage{wrapfig}

\usepackage{booktabs}
\usepackage{multirow}
\usepackage{makecell}
\usepackage{graphicx}
\usepackage[table]{xcolor}

\usepackage{xcolor}
\usepackage{pifont}
\usepackage{float}

\newcommand{\cmark}{\textcolor{green!60!black}{\ding{51}}}
\newcommand{\xmark}{\textcolor{red!70!black}{\ding{55}}}

\newcommand{\best}[1]{\textbf{\textcolor{blue!70!black}{#1}}}
\newcommand{\second}[1]{\underline{#1}}

\newtheorem{proposition}{Proposition}
\newtheorem{remark}{Remark}
\newcommand{\bS}{\mathbf{S}}
\newcommand{\bq}{\mathbf{q}}
\newcommand{\bk}{\mathbf{k}}
\newcommand{\bv}{\mathbf{v}}
\newcommand{\bx}{\mathbf{x}}
\newcommand{\bo}{\mathbf{o}}
\newcommand{\bz}{\mathbf{z}}
\newcommand{\bW}{\mathbf{W}}

\newcommand{\bt}{\mathbf{t}}

\newcommand{\bgamma}{\boldsymbol{\gamma}}

\newcommand{\cL}{\mathcal{L}}

\newcommand{\R}{\mathbb{R}}

\newtheorem{theorem}{Theorem}

\newtheorem{corollary}[theorem]{Corollary}

\renewcommand{\acksection}{\section*{Acknowledgments}}

\title{HorizonStream: Long-Horizon Attention for Streaming 3D Reconstruction}

\author{
\normalfont
Chong Cheng\textsuperscript{1,2}
\quad
Peilin Tao\textsuperscript{2,3}
\quad
Nanjie Yao\textsuperscript{1}
\quad
Guanzhi Ding\textsuperscript{1}
\quad
Xianda Chen\textsuperscript{4}
\quad
Yuansen Du\textsuperscript{2}
\\
Xiaoyang Guo\textsuperscript{2}
\quad
Wei Yin\textsuperscript{2}
\quad
Weiqiang Ren\textsuperscript{2}
\quad
Qian Zhang\textsuperscript{2}
\quad
Zhengqing Chen\textsuperscript{2,\ensuremath{\ddagger}}
\quad
Hao Wang\textsuperscript{1,\ensuremath{\dagger}}
\\[0.5em]
\textsuperscript{1}HKUST(GZ)
\quad
\textsuperscript{2}Horizon Robotics
\quad
\textsuperscript{3}CASIA
\quad
\textsuperscript{4}CSU
\\[0.3em]
\ensuremath{^\dagger} Corresponding author
\quad
\ensuremath{^\ddagger} Project lead
}

\begin{document}

\maketitle
% ====================================================================
% ABSTRACT
% ====================================================================
\begin{abstract}
Online 3D reconstruction requires estimating camera pose and scene geometry under strict causal and bounded-memory constraints. Existing methods often suffer from drift, jitter, or collapse on long sequences. We trace these failures to a fundamental mismatch. Streaming geometry is inherently temporally heterogeneous, with evidence ranging from short-lived correspondences to persistent global scale. However, current architectures impose uniform and pathological influence patterns. For example, sliding windows enforce hard cutoffs, while ungated recurrence and causal attention cause cache saturation and spike-like attention sinks. To resolve this, we formalize geometric propagation as an \emph{evidence influence kernel} and propose HorizonStream, a long-horizon Transformer that explicitly factorizes this kernel. For the long-range temporal factor, Geometric Linear Attention learns channel-wise decay rates to enable bounded, multi-timescale propagation of geometric evidence. For the short-range spatial factor, Gated Window Attention (GWA) with Spatiotemporal RoPE performs reliable 3D matching while suppressing attention sinks. Finally, Metric Readout Tokens recover stable scale and rigid pose directly from the persistent geometric state. Extensive experiments show that HorizonStream, trained on only 48-frame clips, generalizes stably to sequences exceeding \(10{,}000\) frames with constant memory and linear time, achieving state-of-the-art streaming 3D reconstruction performance. Project Page: \url{https://3dagentworld.github.io/horizonstream/}

\end{abstract}

% ====================================================================
% 1. INTRODUCTION
% ====================================================================
\section{Introduction}
\label{sec:intro}

Online 3D reconstruction from streaming video is a core capability for robotics, autonomous driving, and embodied intelligence, requiring causal, bounded-memory estimation of camera pose and scene geometry. Optimization-based methods~\citep{schonberger2016colmap,teed2021droid,campos2021orbslam3,teed2023dpvo,cheng2025graphguidedscenereconstructionimages,song2024gvkfgaussianvoxelkernel} maintain explicit geometric states, but rely on iterative optimization and have limited throughput. Recent offline feed-forward methods~\citep{wang2024dust3r,leroy2024mast3r,wang2025vggt,shen2025fastvggt,deng2026vggtlongchunkitloop,zhang2026loger,hu2025vggt4dminingmotioncues,cheng2025reggsunposedsparseviews,fang2026dens3r,gao2026more} achieve high accuracy, but use full attention and access future frames, violating online causality.

Strictly causal streaming 3D reconstruction still degrades on long sequences~\citep{cheng2026longstream}. Methods often suffer from collapse, pose jitter, and scale instability. This occurs because existing architectures organize history purely by recency. However, recency is a poor proxy for geometric relevance in 3D, as streaming geometry is inherently temporally heterogeneous. Recent evidence may already be invalid, while older evidence can remain reliable. Therefore, we view the reconstruction process as aggregating diverse types of geometric \textbf{\textit{evidence}}. This evidence has vastly different lifetimes. For example, local 2D-3D correspondences are short-lived, which quickly become invalid due to motion. In contrast, global scale and scene structures are persistent, which must remain reliable over long horizons. Yet, existing architectures impose a uniform propagation rule on all evidence. The key question is: \emph{how can we apply the correct temporal influence range for each type of geometric evidence?}

To answer this, we further formalize the temporal propagation of geometric information through an \textbf{\textit{evidence influence kernel}}. We define this kernel as a spatio-temporal weight function, which determines how much past geometric evidence should influence the current reconstruction state. Under this formulation, we find that existing methods inadvertently induce pathological kernels, as shown in Fig. \ref{teaser}. Sliding windows~\citep{lan2026stream3r,zhuo2026streamvggt} impose a hard-cutoff box kernel, which may prematurely discard useful past evidence. Refresh mechanisms~\citep{cheng2026longstream,cheng2025unposed3dgsreconstructionprobabilistic} create blockwise discontinuous kernels. Causal softmax attention~\citep{lingbotmap2026} degenerates into spike-like attention sinks, which focus on irrelevant early tokens. Ungated recurrence~\citep{chen2025ttt3r,wang2025cut3r} forms a heavy-tailed kernel with unbounded error accumulation. As sequences grow longer, these pathological kernels are repeatedly amplified. This causes cache saturation, early-token dominance, and severe geometric drift.

Consequently, current geometric transformer memory designs occupy two extremes of a retention spectrum. Sliding windows force immediate forgetting. Full-attention methods retain everything permanently. Both extremes lack a bounded, flexible temporal form. Instead, a proper approach should learn continuous retention rates tailored to each geometric channel.

\begin{figure}
\centering
\includegraphics[width=\textwidth]{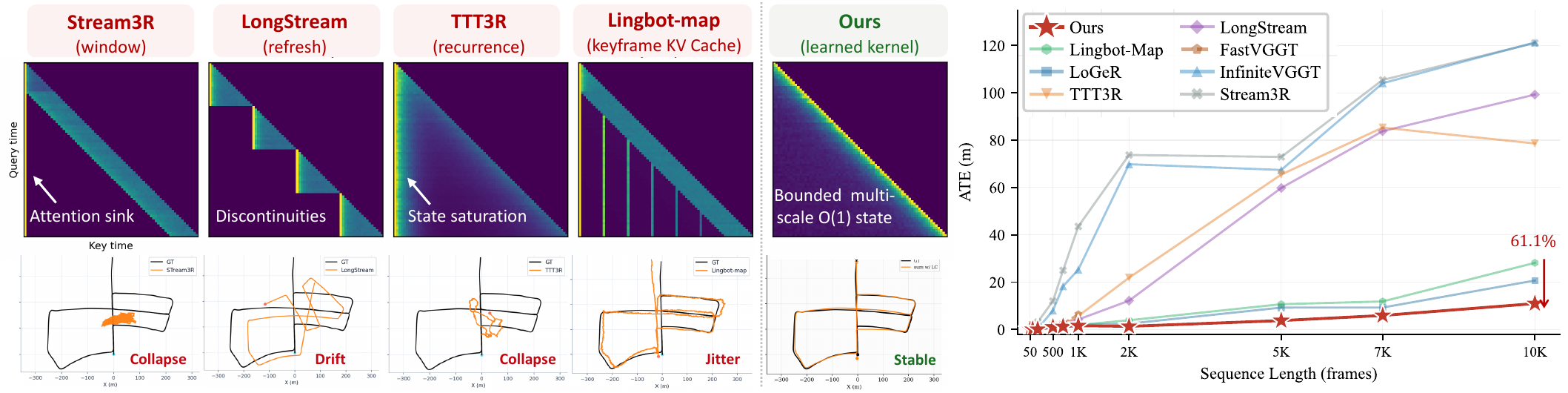}
\vspace{-6pt}
\caption{
Geometric evidence influence patterns on KITTI and long-sequence scaling on VBR.
Prior streaming methods impose hard cutoffs, refresh discontinuities, heavy-tailed states, or spiky KV caches, leading to pose degradation, jitter, or collapse.
HorizonStream learns a bounded multi-scale kernel with an \(O(1)\) recurrent geometric state and maintains stable ATE up to \(10\mathrm{K}\) frames.}
\label{teaser}
\vspace{-15pt}
\end{figure}

To this end, we propose \textbf{HorizonStream}, a long-horizon Transformer that explicitly instantiates this kernel factorization. For the long-range temporal factor, \textbf{Geometric Linear Attention} maintains a bounded \(O(1)\) recurrent state derived from a discounted geometric objective. By learning channel-wise exponential decay rates, it enables stable multi-timescale evidence propagation across windows. For the short-range spatial factor, \textbf{Gated Window Attention} performs 3D content matching within the local window. It uses head-wise reliability gates to filter noisy correspondences and suppress attention sinks, while spatiotemporal RoPE provides relative 3D space-time position bias. Finally, to satisfy the metric invariance constraint, Metric Readout Tokens (MRT) and relative pose fusion recover stable scale and rigid pose directly from the persistent recurrent state.

Since the proposed kernel is local and bounded, it defines a sequence-length-independent propagation rule that can be repeatedly applied to arbitrary-length streams. Experiments on multiple datasets show that HorizonStream, trained on only 48-frame clips, generalizes stably to tens of thousands of frames without pose degradation and outperforms all streaming 3D reconstruction methods.

Our contributions are:

\begin{itemize}[leftmargin=*,nosep]

\item We formalize streaming 3D reconstruction via a geometric evidence influence kernel. This view unifies common long-sequence failures as pathological kernel shapes, i.e., hard cutoffs, discontinuities, attention sinks, and cache saturation.

\item We propose HorizonStream, a constrained kernel-decomposition architecture. Geometric Linear Attention provides bounded multi-timescale propagation across windows; Gated Window Attention with Spatiotemporal RoPE enables content-aware 3D matching within windows; MRT with relative pose fusion preserves metric scale and rigid pose.

\item Experiments on multiple datasets show that HorizonStream, trained only with 48-frame batches, generalizes to sequences over \(10{,}000\) frames with constant memory and linear time, achieving state-of-the-art streaming 3D reconstruction performance.

\end{itemize}

\begin{figure}
\centering
\includegraphics[width=0.9\textwidth]{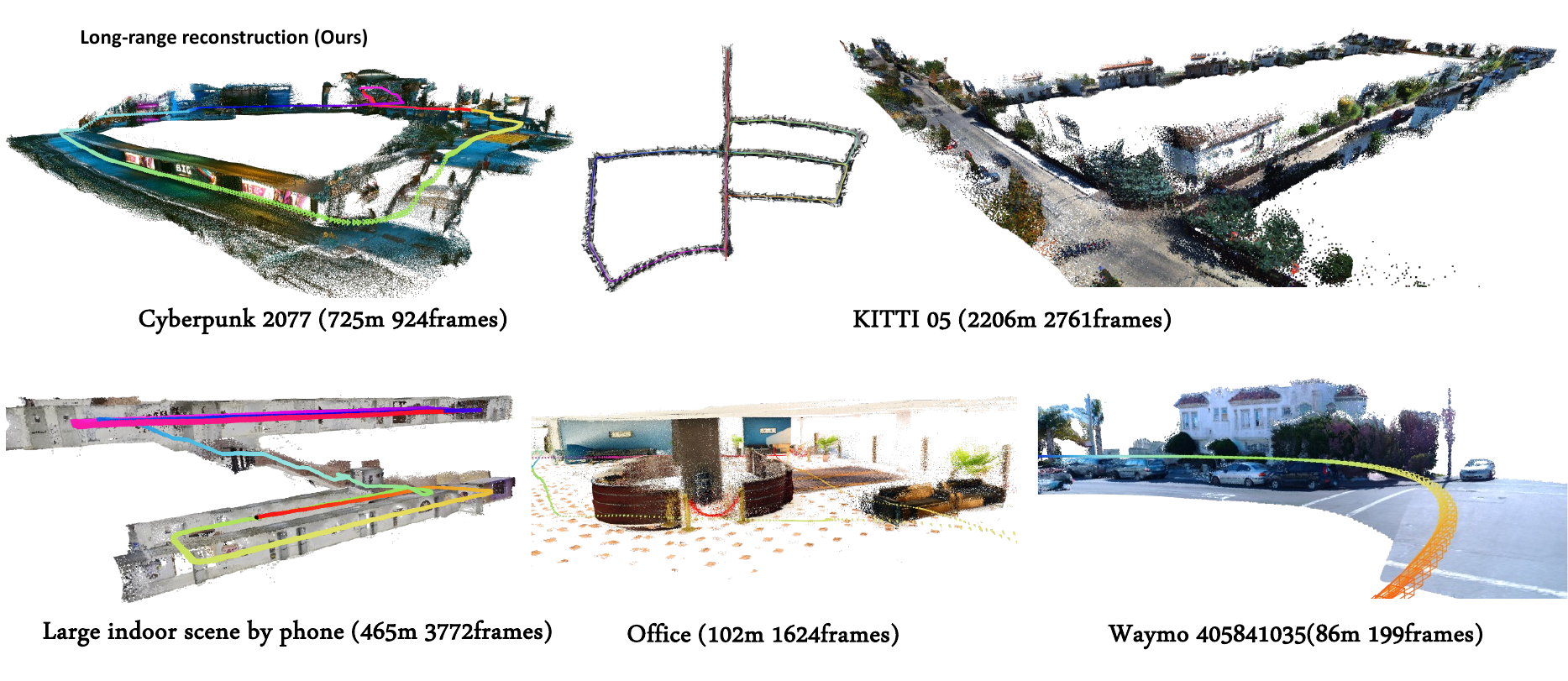}
\vspace{-12pt}
\caption{Visualization of long-range streaming 3D reconstruction across diverse scenes. Our method maintains stable trajectories and coherent geometry over sequences ranging from hundreds to thousands of frames in outdoor, indoor, and game environments.}
\label{fig:teaser_vis}
\vspace{-10pt}
\end{figure}

% ====================================================================
% 2. RELATED WORK
% ====================================================================
\section{Related Work}
\label{sec:related}

\noindent \textbf{Optimization-based visual SLAM and odometry.}
Classical visual SLAM systems such as ORB-SLAM3~\citep{campos2021orbslam3} maintain explicit geometric maps using keyframe management, local optimization, and loop closure. DROID-SLAM and DPVO~\citep{teed2021droid,teed2023dpvo} combine learned correspondence estimation with iterative bundle adjustment. MASt3R-SLAM and VGGT-SLAM~\citep{murai2025mast3rslam,maggio2025vggtslam} integrate feed-forward reconstruction priors into tracking, map fusion or submap alignment, and global optimization.

\noindent \textbf{Offline feed-forward 3D reconstruction.}
DUSt3R~\citep{wang2024dust3r,yu2025rgbonlygaussiansplattingslam} and MASt3R~\citep{leroy2024mast3r,cheng2025outdoormonocularslamglobal} predict dense geometry from image pairs. This paradigm extends to sequences via spatial memory in Spann3R and MonST3R~\citep{wang2025spann3r,zhang2025monst3r}, and to arbitrary image collections with a geometry-aware Transformer in VGGT~\citep{wang2025vggt}. FastVGGT~\citep{shen2025fastvggt} reduces inference memory by reusing attention maps. VGGT-Long~\citep{deng2026vggtlongchunkitloop} and LoGeR~\citep{zhang2026loger} scale to longer inputs through chunk-wise processing or accumulated weights. However, they usually rely on full attention within chunks and chunk stitching lacks cross-chunk dependency, causing temporal discontinuities.

\noindent \textbf{Online feed-forward 3D reconstruction.}
Recent methods adapt feed-forward reconstruction to causal streams~\citep{yu2026vigeo}. STream3R and StreamVGGT~\citep{lan2026stream3r,zhuo2026streamvggt} use causal masks and sliding-window attention, retaining only local-window context. CUT3R and TTT3R~\citep{wang2025cut3r,chen2025ttt3r} add persistent recurrent states, Point3R~\citep{wu2025point3r} maintains spatial pointer memory, InfiniteVGGT~\citep{yuan2026infinitevggt} prunes the KV cache, and Lingbot-map~\citep{lingbotmap2026} extends context with keyframe memory. These designs enable cross-window information transfer but rely on fixed or write-only temporal mechanisms and still suffer from jitter, pose degradation, and disordered geometry on long sequences. LongStream~\citep{cheng2026longstream} attributes long‑sequence degradation to attention sink and state saturation~\citep{xiao2024streamingllm,gu2025sinkemerge}, but its periodic cache refresh discards accumulated context at each boundary, weakening long‑range revisit.

Therefore, we argue that a better online 3D reconstruction pipeline requires a bounded and multi-timescale control over geometric evidence influence. HorizonStream learns channel-wise propagation scales to preserve useful long-range geometry and down-weight stale evidence without cache reset.
% ====================================================================
% 3. METHOD
% ====================================================================

\section{Method}
\label{sec:method}
% ====================================================================

\noindent \textbf{Overview.}
Fig.~\ref{main-main} shows the HorizonStream framework. The model processes the most recent \(W\) frames causally and maintains an \(O(1)\) geometric state for cross-window structure and scale. Gated Window Attention with Spatiotemporal RoPE handles within-window matching, Geometric Linear Attention performs cross-window propagation, and Metric Readout Tokens recover scale.

\subsection{Problem Formulation}

Given an RGB video, streaming 3D reconstruction predicts pose \(\hat{\mathbf{T}}_t\in SE(3)\) and dense depth \(\hat{D}_t\) online from past observations and a bounded state. We describe how past evidence affects the current reconstruction with a geometric evidence influence kernel \(K(t,i)\), which maps evidence at time \(i\) to its contribution at time \(t\).

\begin{figure}[t]
\centering
\includegraphics[width=\textwidth]{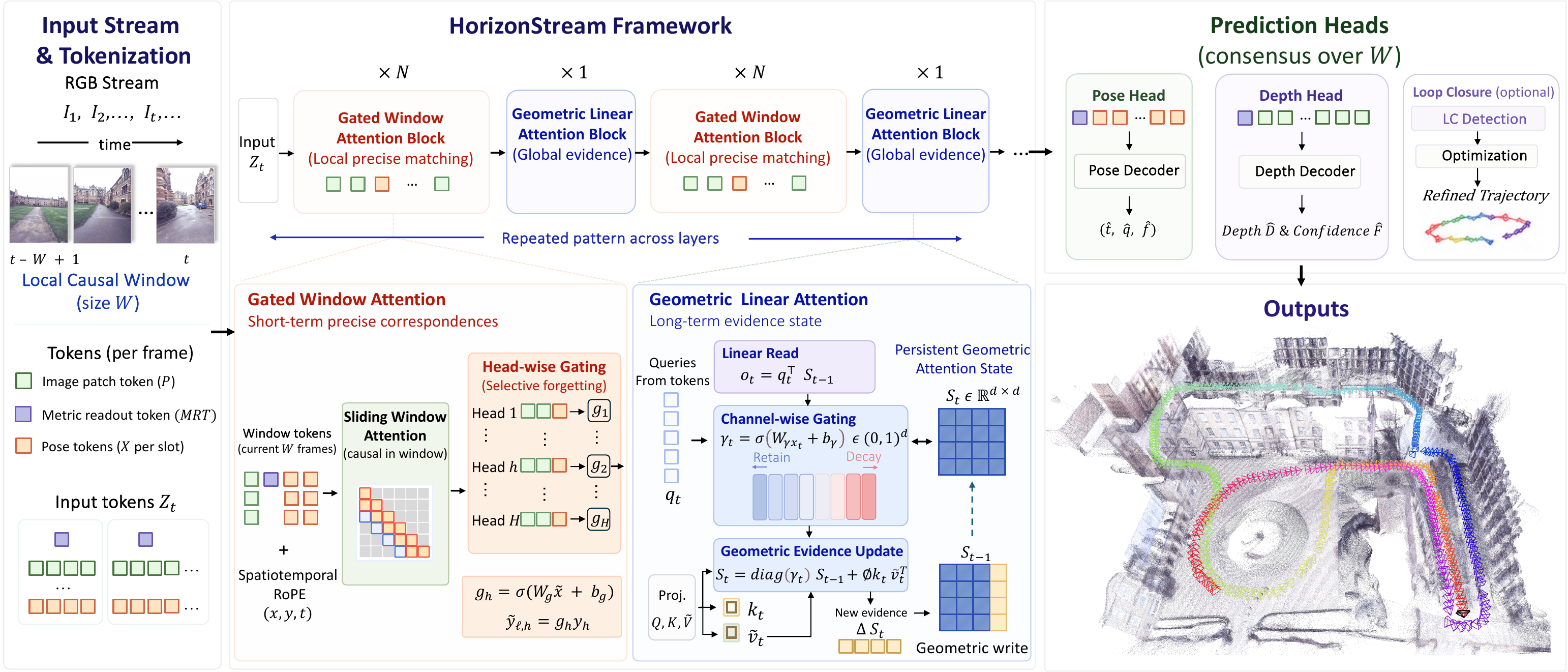}
\vspace{-8pt}
\caption{Overview of \textbf{HorizonStream}.
Given an RGB stream, the model causally processes the most recent \(W\) frames. Gated Window Attention handles local matching, Geometric Linear Attention propagates long-range geometry with an \(O(1)\) recurrent geometric state, and Metric Readout Tokens recover stable scale and pose. An optional loop-closure module refines the trajectory.}
\label{main-main}
\vspace{-15pt}
\end{figure}
A valid geometric evidence influence kernel must solve three core problems: 1) select reliable local correspondences based on spatial content, 2) ensure bounded multi-timescale propagation to prevent state accumulation while respecting diverse evidence lifetimes, and 3) preserve scale and rigid pose. 

To systematically address these requirements, we decouple the influence mechanism into a spatio-temporal kernel factorization augmented by a metric readout. We factorize the kernel as:
\begin{equation}
\label{eq:kernel_factorization}
K(t,i)
=
K_{\mathrm{spatial}}(t,i)
\cdot
K_{\mathrm{time}}(t,i).
\end{equation}

This factorization explicitly maps the three problems to dedicated computational components. First, \(K_{\mathrm{spatial}}\) addresses spatial content-awareness (Problem 1). It uses image content and 3D proximity to select reliable short-range evidence. Second, \(K_{\mathrm{time}}\) addresses bounded multi-timescale propagation (Problem 2). It uses channel-wise exponential decay to keep long-range influence bounded while allowing different geometric channels to propagate over distinct temporal horizons. Finally, Metric Readout Tokens read metric information from the recurrent state to recover stable scale and rigid pose (Problem 3).

Together, these components form a complete, strictly causal streaming architecture. We now detail how this theoretical framework is instantiated into our network architecture. Section~\ref{subsec:ggla} introduces Geometric Linear Attention to model the temporal factor \(K_{\mathrm{time}}\). Section~\ref{sec:local} introduces Gated Window Attention to model the spatial factor \(K_{\mathrm{spatial}}\). Analysis of why open-form operators fail these constraints is provided in Appendix~\ref{app:softmax_dilution} and~\ref{app:theory}.

\subsection{Geometric Linear Attention}
\label{subsec:ggla}

The long-range temporal factor functions as an online geometric estimator over key-value encoded geometric evidence, including correspondence, motion, structure, and scale cues. It summarizes this evidence in a bounded cross-window state, revises stale information, and preserves long-lived geometry. We formulate this through a discounted geometric state-estimation objective:
\begin{equation}
\label{eq:objective}
\mathcal{J}_t(\bS)=
\sum_{i=1}^{t}
\Bigl(\prod_{j=i+1}^{t}\gamma_j\Bigr)
\|\bS^\top \bk_i-\bv_i\|_2^2,
\qquad
K_{\mathrm{time}}(t,i)=\prod_{j=i+1}^{t}\gamma_j .
\end{equation}
Here, \(\bS \in \mathbb{R}^{d \times d}\) is the recurrent geometric state. The vectors \(\bk_i\) and \(\bv_i\) are the key and value encoding the geometric evidence at time \(i\). The variable \(\gamma\) acts as a learned gating factor for information retention. Specifically, \(\gamma_t\) denotes the retention rate at time index \(t\), and \(\gamma_j\) represents the intermediate retention rate at a specific step \(j\) within the cumulative product. With \(\gamma_t\equiv1\), evidence never decays. This causes heavy-tailed accumulation and state contamination. With \(\bar{\gamma}=\sup_t|\gamma_t|<1\), the influence of stale evidence is strictly bounded:
\begin{equation}
\label{eq:stability}
\Bigl\|
\bq_t^\top
\Bigl(\prod_{j=1}^{t}\gamma_j\Bigr)
\bS_0
\Bigr\|
\leq
\|\bq_t\| \cdot \|\bS_0\|_F \cdot \bar{\gamma}^{\,t}
\to 0 .
\end{equation}
In this bound, \(\bq_t\) is the query vector at time \(t\), and \(\bS_0\) is the initial state. The term \(\|\cdot\|_F\) denotes the Frobenius norm. Thus, discounting closes the open-form temporal influence that causes unbounded accumulation.

\noindent \textbf{Online state update.}
The objective admits the recursive form
\begin{equation}
\label{eq:recursive_objective}
\mathcal{J}_t(\bS)
=
\gamma_t\mathcal{J}_{t-1}(\bS)
+
\|\bS^\top\bk_t-\bv_t\|_2^2 .
\end{equation}
This principle yields a fixed-state attention update:
\begin{equation}
\label{eq:linear_update}
\bS_t=
\gamma_t\bS_{t-1}
+
\phi(\bk_t)\tilde{\bv}_t^\top,
\qquad
\bo_t=\bq_t^\top\bS_t .
\end{equation}
Here \(\bS_t\in\mathbb{R}^{d\times d}\) summarizes cross-window reconstruction evidence, \(\phi(\bk_t)\) maps keys into the linear attention feature space, and \(\tilde{\bv}_t\) denotes the value update written into the state.

\noindent \textbf{Channel-wise geometric retention.}
The scalar retention factor \(\gamma_t\) assigns a single lifetime to all evidence, which is insufficient for streaming geometry: local correspondences are short-lived, motion cues persist over moderate horizons, scene structure should survive across windows, and metric scale must remain stable over long sequences. We therefore replace \(\gamma_t\) with a channel-wise retention vector:
\begin{equation}
\label{eq:gate_update}
\bgamma_t=\sigma(\bW_\gamma\bx_t+\mathbf{b}_\gamma)\in(0,1)^d,
\qquad
\bS_t=
\mathrm{diag}(\bgamma_t)\bS_{t-1}
+
\phi(\bk_t)\tilde{\bv}_t^\top .
\end{equation}
Each channel \(c\) then has its own temporal influence factor and effective retention horizon:
\begin{equation}
\label{eq:channel_kernel_horizon}
K_{\mathrm{time}}^{(c)}(t,i)
=
\prod_{j=i+1}^{t}\gamma_j^{(c)},
\qquad
\tau^{(c)}
=
-\frac{1}{\log \bar{\gamma}^{(c)}} .
\end{equation}
Low-\(\gamma\) channels adapt rapidly to incoming observations, while high-\(\gamma\) channels preserve information over longer horizons. Together, they provide a learned spectrum of evidence lifetimes. Geometric cues can be represented across multiple retention bands, consistent with the distributed metric-related signal observed in our probing analysis in Appendix~\ref{app:probe}.

\noindent \textbf{Relation to TTT and linear attention.}
Eq.~\eqref{eq:gate_update} admits an online-learning interpretation: the state adapts to incoming geometric evidence, similar to Test-Time Training (TTT). Explicit per-frame TTT optimization is costly for ultra-long streams, while TTT with KV binding admits an equivalent linear-attention form~\citep{liu2026tttla}. This links online adaptation to efficient recurrent attention and places our update within the family of gated linear attention mechanisms~\citep{katharopoulos2020linear,yang2024gateddeltanet,kimilinear2025}.

HorizonStream achieves this online recurrent form through a geometric state \(\bS_t\) and channel-wise retention \(\bgamma_t\): \(\bS_t\) summarizes cross-window reconstruction evidence, while \(\bgamma_t\) controls the temporal influence of each geometric channel. This yields an adaptive, efficient, and bounded recurrent update for long-range geometric propagation. Appendix~\ref{app:softmax_dilution} analyzes the long-sequence degradation of causal softmax attention and ungated recurrence.
\subsection{Gated Window Attention}
\label{sec:local}

Geometric Linear Attention propagates compressed cross-window evidence, but accurate local reconstruction still requires fine-grained correspondences within each window. We instantiate the short-range spatial factor \(K_{\mathrm{spatial}}\) with Gated Window Attention, which selects local evidence using image content and relative 3D layout before it enters the long-range state.

\noindent \textbf{Head-wise output gating.}
To make the spatial kernel robust to sink-like concentration and noisy matches, we assign each attention head a reliability gate~\citep{qiu2025gatedattn}. For head \(h\),
\begin{equation}
\label{eq:headgate}
g_h = \sigma(\bW_g \bar{\bx} + b_g),
\qquad
\tilde{\mathbf{y}}_h = g_h \cdot \mathbf{y}_h ,
\end{equation}
where \(\bar{\bx}\) is the mean-pooled window feature, \(\mathbf{y}_h\) is the head output, \(\sigma\) is the sigmoid function, and \(\bW_g, b_g\) are learnable projection parameters. The gate downweights unreliable heads and preserves heads that support local matching.

\noindent \textbf{Spatiotemporal RoPE.} We extend RoPE~\citep{su2024rope} to three axes (time, height, and width) to encode relative spatiotemporal layout. For a patch at frame \(t\) and spatial location \((y,x)\), we set \(\pi=(t+1,y+1,x+1)\), split query and key vectors into three parts, and rotate each part along one axis. This makes attention depend on relative space-time offsets. We periodically reset the temporal index to avoid unbounded positional growth, while MRT and pose tokens use \(\pi=(0,0,0)\). Together, gating controls head reliability and Spatiotemporal RoPE supplies relative geometric structure.

\noindent \textbf{Metric Readout Tokens (MRT) and relative pose fusion.}
Long streaming reconstruction requires metric scale and pose to remain consistent across windows. Inspired by scale-token and metric-prediction designs~\citep{cheng2026longstream,keetha2026mapanythinguniversalfeedforwardmetric}, MRT participates in Geometric Linear Attention and aggregates metric-related information from the recurrent state across the learned retention spectrum, extending metric readout from local context to sequence-level evidence.

Each frame includes a learned Metric Readout Token \(\bz^{\mathrm{metric}}\). A scale head predicts \(\hat{s}=\exp(g(\bz^{\mathrm{metric}}))\), which rescales translation and depth:
\begin{equation}
\label{eq:scale}
\hat{\bt}=\hat{s}\cdot\hat{\bt}^{\,\mathrm{raw}},
\qquad
\hat{D}=\hat{s}\cdot\hat{D}^{\mathrm{raw}} .
\end{equation}

For pose, we use relative pose fusion over pose tokens in the local window. A transformer head jointly attends to these tokens and estimates a consensus relative pose for the current frame with respect to the window context. This avoids relying on sequential keyframe chaining~\citep{cheng2026longstream}, where composition errors accumulate over long rollouts. Depth is produced by a DPT head with scale injection.
\subsection{Architecture}
\label{sec:arch}

\noindent \textbf{Backbone.} HorizonStream uses a ViT-L backbone initialized from VGGT~\citep{wang2025vggt} and DINOv2~\citep{oquab2024dinov2learningrobustvisual}. Each frame contains image patch tokens, pose tokens, and a Metric Readout Token. The backbone alternates frame blocks and global blocks: frame blocks perform intra-frame self-attention, while global blocks adopt a hybrid temporal design that combines Gated Window Attention for dense intra-window tracking with Geometric Linear Attention layers interleaved at specific depths for cross-window memory updates.

\noindent \textbf{Training objective.}
The model is supervised with pose, depth, and scale losses:
\begin{equation}
\cL = \lambda_{pose}\,\cL_{pose} + \lambda_{depth}\,\cL_{depth} + \lambda_{scale}\,\cL_{scale}.
\end{equation}
Translation and depth are normalized by geometric scale factors. Depth loss is SmoothL1 with confidence weighting. Scale loss applies only on metric-scale samples.

\noindent \textbf{Loop closure.}
To correct long-term accumulated drift during inference, an optional loop-closure module improves global revisit consistency. Inspired by VGGT-Long~\citep{deng2026vggtlongchunkitloop}, we retrieve revisited frame pairs using stored compact global descriptors pooled from DINOv2 features. The retrieved candidates are re-fed into the network to estimate local geometric corrections. These are then converted into loop constraints to optimize the final global trajectory via pose graph optimization.

% ====================================================================
\begin{figure}[t]
\centering
\includegraphics[width=\textwidth]{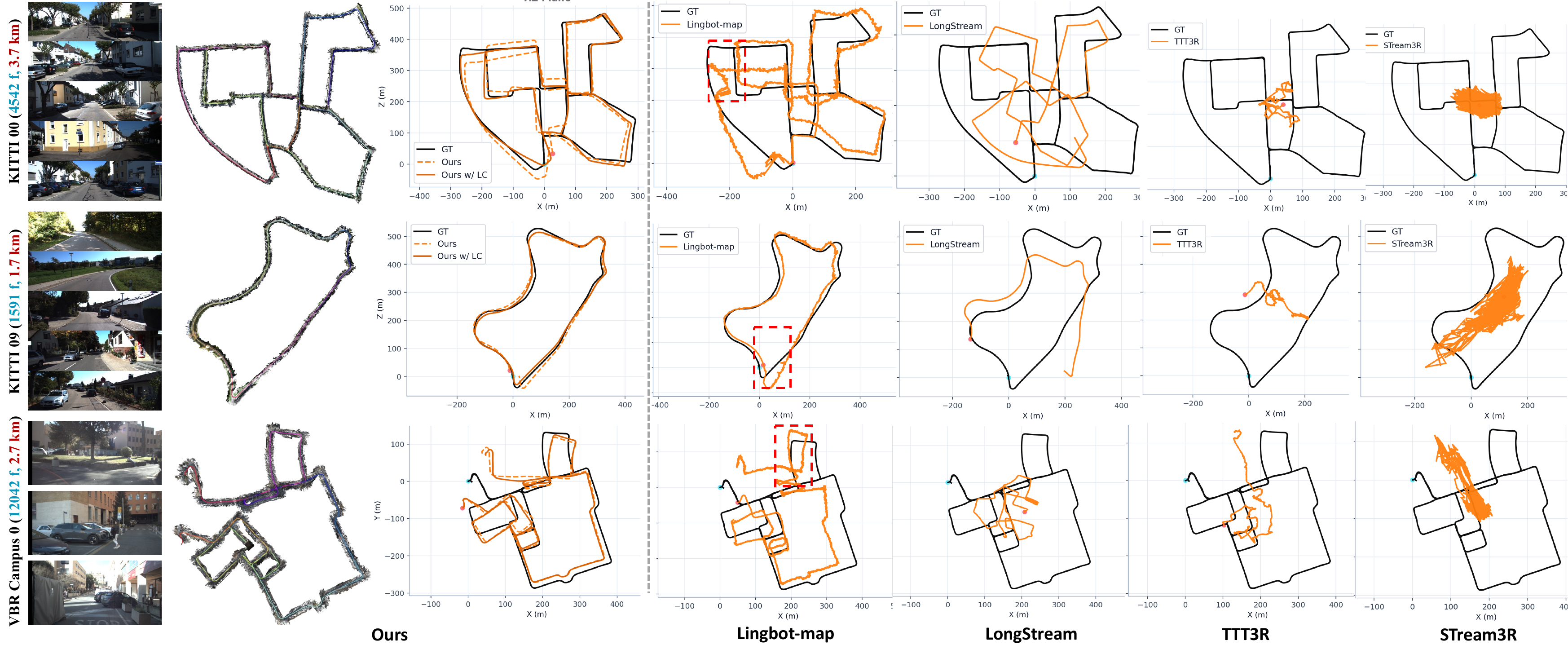}
\vspace{-18pt}
\caption{\textbf{Qualitative comparison} on long-sequence 3D reconstruction.
As sequence length grows, existing methods show pose degradation, drift, or collapse.
Lingbot-map exhibits progressively stronger pose jitter over longer rollouts, while HorizonStream maintains stable pose estimation. For HorizonStream, dashed lines denote results without loop closure.}
\label{fig:vis1}
%\vspace{-15pt}
\vspace{-10pt}
\end{figure}

% ====================================================================
% 4. EXPERIMENTS
% ====================================================================
\begin{table*}[t]
\centering
\caption{
Quantitative comparison on KITTI. We report mean ATE; ``--'' denotes OOM or repeated tracking failure, and LoGeR$^{*}$ denotes optimization-based LoGeR. Refresh/no-refresh variants degrade on long sequences, while trained with 48-frame batches, HorizonStream outperforms all streaming methods and approaches or surpasses offline methods with or without loop closure (LC).
}
\vspace{-5pt}
\label{tab:kitti_ate}
\setlength{\tabcolsep}{2.8pt}
\renewcommand{\arraystretch}{1.10}
\newcommand{\seqinfo}[2]{%
{\scriptsize\textcolor{cyan!70!black}{#1 fr.}}\\[-1pt]
{\scriptsize\textcolor{red!70!black}{#2 km}}%
}
\resizebox{0.8\textwidth}{!}{
\begin{tabular}{c l |ccccccccccc|c}
\toprule
& \multirow{3}{*}{Methods}
& \multicolumn{12}{c}{KITTI ATE $\downarrow$} \\
\cmidrule(lr){3-13}
&
& \makecell{\textbf{00}\\\seqinfo{4542}{3.7}}
& \makecell{\textbf{01}\\\seqinfo{1101}{2.5}}
& \makecell{\textbf{02}\\\seqinfo{4661}{5.1}}
& \makecell{\textbf{03}\\\seqinfo{801}{0.6}}
& \makecell{\textbf{04}\\\seqinfo{271}{0.4}}
& \makecell{\textbf{05}\\\seqinfo{2761}{2.2}}
& \makecell{\textbf{06}\\\seqinfo{1101}{1.2}}
& \makecell{\textbf{07}\\\seqinfo{1101}{0.7}}
& \makecell{\textbf{08}\\\seqinfo{4071}{3.2}}
& \makecell{\textbf{09}\\\seqinfo{1591}{1.7}}
& \makecell{\textbf{10}\\\seqinfo{1201}{0.9}}
& Avg. \\
\midrule

\multirow{6}{*}{\rotatebox[origin=c]{90}{Opt.-based}}
& MASt3R-SLAM & -- & 530.37 & -- & 18.87 & 88.98 & 159.43 & 92.00 & -- & 263.75 & -- & 153.07 & 186.64 \\
& VGGT-SLAM & -- & 607.16 & -- & 169.83 & 13.12 & -- & -- & -- & -- & -- & 211.82 & 250.48 \\
& COLMAP & 139.12 & 3.83 & 71.99 & 1.46 & 112.77 & 20.37 & 10.95 & 7.80 & 21.72 & 21.19 & 4.52 & 37.79 \\
& MASt3R-SfM & -- & 463.52 & -- & 15.80 & 41.44 & 150.39 & 136.14 & 71.69 & -- & 176.36 & 69.50 & 140.60 \\
& DPVO & 113.11 & 16.60 & 113.01 & 2.46 & 0.98 & 59.34 & 55.91 & 19.30 & 110.63 & 74.55 & 13.71 & 52.69 \\
& DROID-SLAM & -- & 82.81 & -- & 3.20 & 1.47 & 73.50 & 61.10 & 18.41 & 104.22 & 89.49 & 22.19 & 50.71 \\

\midrule

\multirow{5}{*}{\rotatebox[origin=c]{90}{Offline Fwd.}}
& VGGT-Long & 8.64 & 61.21 & 52.72 & 8.78 & 4.20 & 9.88 & 4.67 & 2.66 & 72.98 & 31.84 & 27.71 & 25.94 \\
& FastVGGT & -- & 705.39 & -- & 62.38 & 10.27 & 157.74 & 124.43 & 69.27 & -- & 190.10 & 194.75 & 189.29 \\
& LoGeR & 54.98 & 36.57 & 36.20 & 4.27 & 1.62 & 33.41 & 11.78 & 13.33 & 22.92 & 17.89 & 8.06 & 21.91 \\
& LoGeR$^{*}$ & 26.19 & 41.26 & 32.21 & 5.02 & 1.62 & 22.65 & 5.49 & 5.04 & 21.96 & 9.03 & 9.44 & 16.35 \\
& LoGeR w/o refresh & 166.05 & 631.14 & 226.65 & 66.09 & 4.55 & 125.16 & 98.32 & 12.38 & 203.24 & 127.28 & 185.19 & 167.82 \\

\midrule

\multirow{10}{*}{\rotatebox[origin=c]{90}{Online Fwd.}}
& CUT3R w/o refresh & 185.89 & 651.52 & 296.98 & 148.06 & 22.17 & 155.61 & 132.54 & 77.03 & 238.39 & 205.94 & 193.39 & 209.78 \\
& CUT3R w/ refresh & 190.38 & 90.59 & 264.39 & 20.40 & 7.31 & 92.25 & 67.54 & 22.48 & 145.08 & 67.42 & 40.00 & 91.62 \\
& TTT3R w/o refresh & 190.93 & 546.84 & 218.77 & 105.28 & 11.62 & 153.12 & 132.94 & 70.95 & 180.57 & 211.01 & 133.00 & 177.73 \\
& TTT3R w/ refresh & 119.94 & 99.59 & 238.07 & 16.83 & 3.98 & 36.38 & 47.20 & 11.62 & 107.33 & 86.96 & 33.58 & 72.86 \\
& STream3R & 190.98 & 681.95 & 301.40 & 158.25 & 102.73 & 159.85 & 135.03 & 90.37 & 261.15 & 216.31 & 207.49 & 227.77 \\
& StreamVGGT & 191.93 & 653.06 & 303.35 & 157.50 & 108.24 & 160.46 & 133.71 & 89.00 & 263.95 & 216.69 & 209.80 & 226.15 \\
& InfiniteVGGT & 167.17 & 533.36 & 272.99 & 149.18 & 58.86 & 127.50 & 100.54 & 78.77 & 196.66 & 199.25 & 138.04 & 183.85 \\
& LongStream & 92.55 & \second{46.01} & 134.70 & \second{3.81} & 1.95 & 84.69 & 23.12 & 14.93 & 62.07 & 85.61 & 21.48 & 51.90 \\
& Lingbot-map & 30.80 & 64.74 & \second{82.29} & \best{2.49} & \second{0.85} & 16.55 & \second{6.27} & 8.92 & \second{39.32} & \best{17.99} & \best{7.96} & 25.29 \\
\cmidrule(lr){2-14}
& \textbf{Ours} & \second{26.40} & \best{20.62} & 84.62 & 5.15 & \best{0.62} & \second{12.82} & \best{4.59} & \second{5.49} & \best{19.49} & 25.73 & \second{11.71} & \second{19.75} \\
& \textbf{Ours w/ LC} & \best{13.91} & \best{20.62} & \best{69.43} & 5.15 & \best{0.62} & \best{6.86} & 6.50 & \best{2.67} & \best{19.49} & \second{23.86} & \second{11.71} & \best{16.44} \\
\bottomrule
\end{tabular}
}
\vspace{-15pt}
\end{table*}

\section{Experiments}
\label{sec:exp}

\subsection{Experimental Setup}
\noindent \textbf{Datasets.}
We evaluate on KITTI~\citep{geiger2012kitti}, vKITTI2~\citep{cabon2020vkitti2}, Oxford Spires~\citep{tao2025oxford}, ScanNet++~\citep{yeshwanth2023scannetpp}, TUM RGB-D~\citep{sturm2012tum}, Waymo Open~\citep{sun2020waymo}, VBR~\citep{brizi2024vbr}, ETH3D~\citep{schops2017eth3d}, and 7Scenes~\citep{shotton20137scenes}.
All sequences are evaluated at full length without subsampling. vKITTI2, 7Scenes, Waymo, and ScanNet++ are included in our training data; Waymo evaluation uses segments not seen during training.
Detailed evaluation splits and per-dataset protocols are in Appendix~\ref{app:eval_data}.

\noindent \textbf{Baselines.}
We compare against three paradigms:
(i) \emph{optimization-based}: COLMAP~\citep{schonberger2016colmap}, DPVO~\citep{teed2023dpvo}/DPVO++, DROID-SLAM~\citep{teed2021droid}, MASt3R-SLAM~\citep{murai2025mast3rslam}, MASt3R-SfM~\citep{leroy2024mast3r}, VGGT-SLAM~\citep{maggio2025vggtslam};
(ii) \emph{offline feed-forward}: VGGT-Long~\citep{deng2026vggtlongchunkitloop}, FastVGGT~\citep{shen2025fastvggt}, LoGeR~\citep{zhang2026loger} (and its optimization variant LoGeR$^*$), Pi3-Chunk~\citep{pi32026};
(iii) \emph{online feed-forward}: CUT3R~\citep{wang2025cut3r}, TTT3R~\citep{chen2025ttt3r}, STream3R~\citep{lan2026stream3r}, StreamVGGT~\citep{zhuo2026streamvggt}, InfiniteVGGT~\citep{yuan2026infinitevggt}, LongStream~\citep{cheng2026longstream}, Lingbot-map~\citep{lingbotmap2026}.
For CUT3R, TTT3R, and LoGeR, we report refresh and no-refresh variants to isolate the effect of periodic state reset.
All baselines are evaluated on full sequences without subsampling using the released code and the default settings. Our inference code and evaluation scripts have been released.

\subsection{Implementation Details}
\label{sec:impl}

\begin{figure}[t]
\centering
\includegraphics[width=\textwidth]{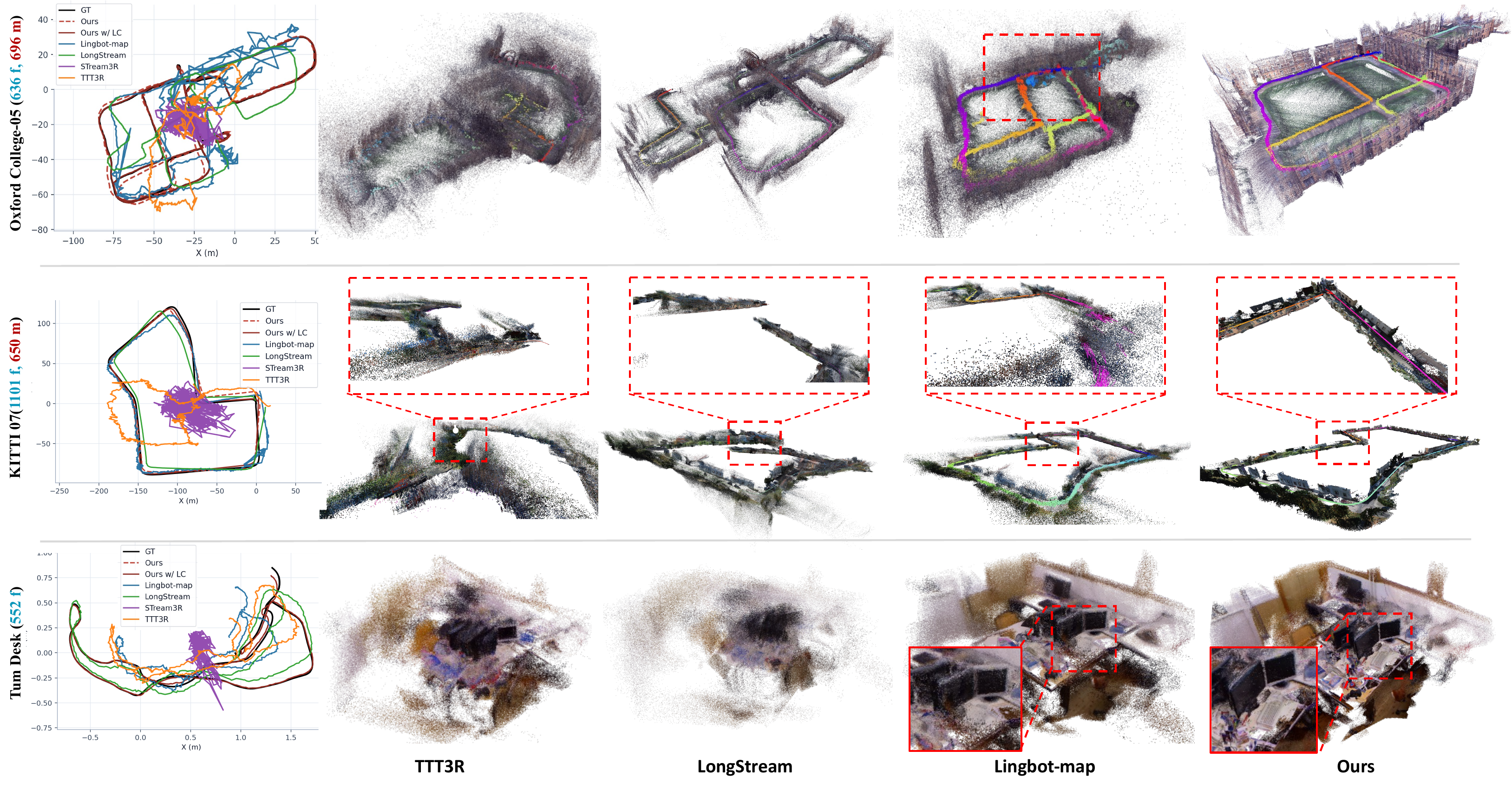}
\vspace{-8pt}
\caption{\textbf{Qualitative comparison} on 3D reconstruction. Left: trajectory. Right: 3D reconstruction. HorizonStream maintains stable geometry. Lingbot-map preserves trajectory direction but exhibits increasing jitter, causing point cloud overlap. For HorizonStream, dashed trajectory lines denote results without loop closure, and solid trajectory lines denote results with loop closure.}
\label{vis2}
\vspace{-12pt}
\end{figure}

\begin{table}
\centering
%\caption{Quantitative comparison across datasets. We report ATE.}
\caption{\textbf{Quantitative comparison} across datasets.
VKITTI2, Waymo and ScanNet++ are in-domain training datasets.
HorizonStream performs strongly in both settings.}
\label{tab:multi_dataset}
\setlength{\tabcolsep}{4pt}
\renewcommand{\arraystretch}{1.12}
\resizebox{0.65\textwidth}{!}{
\begin{tabular}{c l | c | cccccc|c}
\toprule
& \multirow{2}{*}{Method} & \multirow{2}{*}{Calib.-free}
& \multicolumn{6}{c}{ATE (m) $\downarrow$} & \multirow{2}{*}{FPS$\uparrow$}\\
\cmidrule(lr){4-9}
& & 
& VKITTI2 & KITTI & Oxford & ScanNet++ & TUM & Waymo & \\
\midrule

\multirow{6}{*}{\rotatebox[origin=c]{90}{Opt.-based}}
& MASt3R-SLAM & \cmark & 81.55 & 186.64 & 37.73 & 0.47 & 0.08 & 7.63 & 7.40 \\
& VGGT-SLAM & \cmark & 19.23 & 250.48 & 31.00 & 0.29 & 0.12 & 7.43 & 15.80 \\
& COLMAP & \cmark & 9.59 & 37.79 & 15.57 & GT & 0.19 & 25.63 & 0.20 \\
& MASt3R-SfM & \cmark & 49.48 & 140.60 & 32.13 & 1.50 & 0.39 & 3.95 & 0.30 \\
& DPVO++ & \xmark & 0.38 & 52.69 & 34.03 & 0.91 & 0.10 & 1.35 & 19.30 \\
& DROID-SLAM & \xmark & 1.12 & 50.71 & 31.08 & 0.97 & 0.11 & 6.67 & 13.60 \\

\midrule

\multirow{4}{*}{\rotatebox[origin=c]{90}{Offline Fwd.}}
& VGGT-Long & \cmark & 0.91 & 25.94 & 21.90 & 0.13 & 0.08 & 1.78 & 4.80 \\
& FastVGGT & \cmark & 21.52 & 189.29 & 36.58 & 1.56 & 0.42 & 1.28 & 14.20 \\
& LoGeR & \cmark & 1.66 & 21.91 & 18.70 & 0.50 & 0.07 & 0.96 & 16.0 \\
& LoGeR$^{*}$ & \cmark & 2.45 & 16.35 & 15.79 & 0.43 & 0.08 & 0.55 & 9.1  \\

\midrule

\multirow{9}{*}{\rotatebox[origin=c]{90}{Online Fwd.}}
& CUT3R & \cmark & 47.66 & 209.78 & 32.44 & 1.27 & 0.54 & 9.40 & 19.90 \\
& TTT3R & \cmark & 24.18 & 177.73 & 36.21 & 0.55 & 0.31 & 3.49 & 22.00 \\
& STream3R & \cmark & 68.96 & 227.77 & 37.57 & 1.75 & 0.63 & 42.20 & 8.20 \\
& StreamVGGT & \cmark & 68.51 & 226.15 & 37.25 & 1.70 & 0.63 & 45.10 & 19.10 \\
& InfiniteVGGT & \cmark & 58.63 & 183.85 & 31.82 & 1.66 & 0.21 & 20.56 & 5.30 \\
& LongStream & \cmark & 1.61 & 51.90 & 19.82 & \second{0.49} & \second{0.08} & \second{0.74} & 17.10 \\
& Lingbot-map & \cmark & \second{1.30} & 25.29 & 15.46 & 0.52 & \best{0.04} & 1.66 & 11.9 \\
\cmidrule{2-10}
& \textbf{Ours} & \cmark & \best{0.94} & \second{19.75} & \second{9.38} & \best{0.40} & \best{0.04} & \best{0.46} & 13.20 \\
& \textbf{Ours w/ LC} & \cmark & \best{0.94} & \best{16.44} & \best{8.71} & \best{0.40} & \best{0.04} & \best{0.46} & 10.45 \\
\bottomrule
\end{tabular}
}
\vspace{-12pt}
\end{table}

Training mirrors streaming inference: both training stages use 48-frame samples, processed sequentially in 21-frame computational chunks, with the Geometric Linear Attention state propagating sequentially across chunks via a causal window. The pose prediction window is $W{=}10$, so short-term history spans 10 frames. Training proceeds in two stages: Stage~1 on 64 A800 GPUs for 60k iterations, Stage~2 on 64 H20 GPUs for 40k iterations with more long-sequence data. We use AdamW with learning rate $2{\times}10^{-5}$ and cosine schedule with 2000 warmup steps. The reported GPU memory usage excludes the optional loop-closure module. Additional architecture specifications are in Appendix~\ref{app:arch}.

\noindent \textbf{Training data.}
We train on 24 datasets spanning indoor, outdoor driving, large-scale reconstruction, and synthetic environments, 
including ScanNet++~\citep{yeshwanth2023scannetpp}, Hypersim~\citep{roberts2021hypersim}, Replica~\citep{straub2019replica}, 7Scenes~\citep{shotton20137scenes}, ARKitScenes~\citep{baruch2021arkitscenes}, WildRGB-D~\citep{xia2024wildrgbd}, Waymo~\citep{sun2020waymo}, vKITTI2~\citep{cabon2020vkitti2}, Mapillary~\citep{warburg2020mapillary}, MegaDepth~\citep{li2018megadepth}, BlendedMVS~\citep{yao2020blendedmvs}, DL3DV~\citep{ling2024dl3dv}, CO3Dv2~\citep{reizenstein2021co3d}, TartanAir~\citep{wang2020tartanair}, PointOdyssey~\citep{zheng2023pointodyssey}, OmniWorld~\citep{zhou2025omniworldmultidomainmultimodaldataset}, MatrixCity~\citep{li2023matrixcity}, and internal long-sequence data, among others. Training clips use temporal strides from 1 to 8. For unordered image sets, we build pseudo-temporal sequences by traversing the camera graph. Frames are randomly permuted within each chunk with probability 0.2, while the cross-chunk order is preserved. Stage 1 focuses on short-window pose accuracy; Stage 2 incorporates longer source videos while retaining 48-frame training samples, increasing the coverage of long-sequence scenes and revisit patterns. Full list and per-stage sampling ratios are in Appendix~\ref{app:training_data}.

\subsection{Camera Trajectory Estimation}
\label{sec:exp_pose}

% ================= 第二个表格 =================
\noindent \textbf{Long-short sequence generalization.}
Tab.~\ref{tab:kitti_ate},~\ref{tab:multi_dataset}, and~\ref{tab:vbr_ate} report mean ATE for trajectory estimation from indoor scenes to KITTI-scale driving and ultra-long VBR sequences exceeding \(10{,}000\) frames.  On indoor benchmarks, HorizonStream is evaluated on the full sequences without downsampling. It achieves the best overall performance among online methods and remains competitive with offline approaches. As sequence length grows, existing streaming methods show pose degradation, severe jitter, or collapse; Lingbot-map can achieve competitive ATE, but its pose becomes increasingly jittery over longer sequences, as shown in Fig.~\ref{fig:vis1}. HorizonStream remains stable across all sequences.

% Keep analysis and Table 3 together while preserving the regular table float order.
\begin{table}[!htbp]
\noindent
\begin{minipage}[t]{0.32\textwidth}
\raggedright
\noindent \textbf{KV-cache contamination.}
Refresh/no-refresh variants of CUT3R, TTT3R, and LoGeR isolate periodic state reset. Without refresh, all three degrade sharply, indicating temporal-state contamination rather than limited model capacity. HorizonStream avoids periodic refresh by discounting stale evidence and maintaining a bounded geometric state throughout the sequence.
\end{minipage}\hfill
\begin{minipage}[t]{0.65\textwidth}
\vspace{-16pt}
\centering
\caption{\textbf{Quantitative comparison} on VBR.}
% \caption{Quantitative comparison on VBR. We report ATE, where lower is better.}
\label{tab:vbr_ate}
\setlength{\tabcolsep}{3.5pt}
\renewcommand{\arraystretch}{1.12}
\resizebox{\linewidth}{!}{
\begin{tabular}{c l| ccccccc|c}
\toprule
& \multirow{2}{*}{Method}
& \multicolumn{7}{c}{VBR ATE $\downarrow$}
& \multirow{2}{*}{Avg.} \\
\cmidrule(lr){3-9}
&
& \makecell{colosseo\_0\\{\scriptsize\textcolor{cyan!70!black}{8815 fr.}}\\{\scriptsize\textcolor{red!70!black}{1.45 km}}}
& \makecell{campus\_0\\{\scriptsize\textcolor{cyan!70!black}{12042 fr.}}\\{\scriptsize\textcolor{red!70!black}{2.73 km}}}
& \makecell{campus\_1\\{\scriptsize\textcolor{cyan!70!black}{11671 fr.}}\\{\scriptsize\textcolor{red!70!black}{2.95 km}}}
& \makecell{pincio\_0\\{\scriptsize\textcolor{cyan!70!black}{11142 fr.}}\\{\scriptsize\textcolor{red!70!black}{1.27 km}}}
& \makecell{spagna\_0\\{\scriptsize\textcolor{cyan!70!black}{14141 fr.}}\\{\scriptsize\textcolor{red!70!black}{1.56 km}}}
& \makecell{diag\_0\\{\scriptsize\textcolor{cyan!70!black}{10021 fr.}}\\{\scriptsize\textcolor{red!70!black}{1.02 km}}}
& \makecell{ciampino\_1\\{\scriptsize\textcolor{cyan!70!black}{18846 fr.}}\\{\scriptsize\textcolor{red!70!black}{5.20 km}}}
& \\
\midrule

\multirow{6}{*}{\rotatebox[origin=c]{90}{Opt./Offline}}
& VGGT-SLAM & 101.00 & 93.51 & 71.74 & 66.42 & 57.00 & 33.64 & 124.10 & 78.20 \\
& VGGT-Long w/o LC & 81.54 & 118.59 & 98.21 & 53.44 & 46.92 & 30.80 & 170.30 & 85.69 \\
& VGGT-Long & 39.56 & 118.59 & 98.21 & 53.44 & 50.27 & 30.80 & 172.13 & 80.43 \\
& LoGeR & 31.77 & 27.90 & 30.80 & 17.96 & 21.33 & 32.25 & 34.16 & 28.02 \\
& LoGeR$^{*}$ & 55.32 & 13.27 & 16.79 & 9.18 & 18.32 & 29.45 & 34.32 & 25.24 \\
& Pi3-Chunk & 77.09 & 78.50 & 65.77 & 41.99 & 44.76 & 23.81 & 111.72 & 63.38 \\

\midrule

\multirow{6}{*}{\rotatebox[origin=c]{90}{Online Fwd.}}
& CUT3R & 82.63 & 42.25 & 43.16 & 46.65 & 44.62 & 28.62 & 175.83 & 66.25 \\
& TTT3R & 75.52 & 59.44 & 56.55 & 33.87 & 37.33 & \best{18.49} & 173.71 & 64.99 \\
& InfiniteVGGT & 83.91 & 123.65 & 100.00 & 70.73 & 56.25 & 31.58 & -- & 91.60 \\
& LongStream & 72.52 & 100.57 & 105.55 & 43.47 & 59.31 & 32.35 & 131.78 & 77.93 \\
& Lingbot-map & \second{16.70} & \second{23.61} & \second{10.37} & 29.37 & 24.29 & 24.12 & 64.24 & 27.53 \\
\cmidrule{2-10}
& Ours & 37.42 & \best{22.46} & 22.49 & \second{22.63} & \second{23.52} & \second{22.46} & \second{26.10} & \second{25.30} \\
& Ours w/ LC & \best{12.76} & 28.54 & \best{8.49} & \best{17.24} & \best{23.06} & 24.05 & \best{17.76} & \best{18.84} \\

\bottomrule
\end{tabular}
}
\end{minipage}
\vspace{-8pt}
\end{table}

\begin{table}[!htbp]
\centering
% 第一个表格（多视图重建基准测试），占据 65% 的页面宽度
\begin{minipage}[t]{0.65\textwidth}
    \centering
    \caption{\textbf{Quantitative comparison} of CD ($\downarrow$) and F1 ($\uparrow$) on multi-view reconstruction benchmarks.}
    \label{tab:recon}
    \setlength{\tabcolsep}{3.5pt}
    \renewcommand{\arraystretch}{1.12}
    
    % 使用 \linewidth 确保表格填满当前的 minipage
    \resizebox{0.9\linewidth}{!}{
    \begin{tabular}{c l | cc cc cc cc}
    \toprule
    & \multirow{2}{*}{Method}
    & \multicolumn{2}{c}{ETH3D}
    & \multicolumn{2}{c}{Oxford Spires}
    & \multicolumn{2}{c}{7Scenes}
    & \multicolumn{2}{c}{TUM} \\

    \cmidrule(lr){3-4}
    \cmidrule(lr){5-6}
    \cmidrule(lr){7-8}
    \cmidrule(lr){9-10}

    &
    & CD $\downarrow$ & F1@0.25 $\uparrow$
    & CD $\downarrow$ & F1@4 $\uparrow$
    & CD $\downarrow$ & F1@0.25 $\uparrow$
    & CD $\downarrow$ & F1@0.25 $\uparrow$ \\

    \midrule
    \multirow{5}{*}{\rotatebox[origin=c]{90}{Offline/Opt.}}
    & VGGT-Long     & 0.24 & 0.84 & 6.37 & 0.72 & 6.31 & 0.70 & 0.87 & 0.75 \\
    & MASt3R-SLAM   & 0.89 & 0.31 & 14.59 & 0.35 & 6.32 & 0.71 & 0.10 & 0.92 \\
    & VGGT-SLAM     & 0.78 & 0.72 & 11.51 & 0.32 & 6.37 & 0.71 & 0.10 & 0.93 \\
    & FastVGGT      & 0.50 & 0.70 & 7.97 & 0.63 & 5.99 & 0.69 & 0.07 & 0.94 \\
    & LoGeR         & 0.09 & 0.90 & 1.92 & 0.85 & 6.81 & 0.71 & 0.06 & 0.96 \\

    \midrule
    \multirow{8}{*}{\rotatebox[origin=c]{90}{Online}}
    & StreamVGGT    & 1.86 & 0.14 & 15.45 & 0.27 & 6.23 & 0.66 & 0.39 & 0.59 \\
    & STream3R      & 1.81 & 0.14 & 15.44 & 0.26 & 6.31 & \second{0.72} & \second{0.15} & 0.86 \\
    & CUT3R         & 0.41 & 0.60 & 8.22 & 0.41 & 6.35 & 0.48 & 1.51 & 0.32 \\
    & TTT3R         & 0.43 & 0.59 & 9.95 & 0.30 & 6.63 & 0.48 & 0.86 & 0.29 \\
    & InfiniteVGGT  & 0.46 & 0.61 & 9.65 & 0.43 & 6.43 & 0.69 & 0.22 & 0.81 \\
    & LongStream    & 0.77 & 0.55 & \second{6.28} & \second{0.55} & \best{2.26} & 0.64 & 0.23 & 0.67 \\
    & Lingbot-map   & \second{0.37} & \second{0.68} & 8.69 & 0.43 & 6.33 & \second{0.72} & \best{0.08} & \second{0.94} \\
    \cmidrule{2-10}
    & Ours          & \best{0.32} & \best{0.74} & \best{4.97} & \best{0.89} & \second{2.98} & \best{0.93} & \best{0.08} & \best{0.95} \\
    \bottomrule
    \end{tabular}
    }
\end{minipage}\hfill % \hfill 用于在两个表格之间自动填充空白
% 第二个表格（KITTI 深度估计），占据 32% 的页面宽度
\begin{minipage}[t]{0.315\textwidth}
    \centering
    \caption{Video depth estimation results on KITTI.}
    \label{tab:kitti_depth}
    \setlength{\tabcolsep}{4pt}
    \renewcommand{\arraystretch}{1.05}
    
    % 同样使用 \linewidth
    \resizebox{0.9\linewidth}{!}{
    \begin{tabular}{lcc}
    \toprule
    Method & Abs Rel $\downarrow$ & $\delta<1.25$ $\uparrow$ \\
    \midrule
    DUSt3R-GA      & 0.144 & 81.3 \\
    MASt3R-GA      & 0.183 & 74.5 \\
    MonST3R-GA     & 0.168 & 74.4 \\
    VGGT           & \second{0.061} & \best{97.0} \\
    Spann3R        & 0.198 & 73.7 \\
    CUT3R          & 0.118 & 88.1 \\
    Point3R        & 0.136 & 84.2 \\
    \midrule
    StreamVGGT     & 0.173 & 72.1 \\
    STream3R       & 0.080 & 94.7 \\
    InfiniteVGGT   & 0.170 & 78.6 \\
    LoGeR          & 0.090 & 93.0 \\
    LongStream     & 0.120 & 87.0 \\
    Lingbot-map    & 0.098 & 90.7 \\
    \midrule
    Ours           & \best{0.057} & \second{94.8} \\
    \bottomrule
    \end{tabular}
    }
\end{minipage}
\vspace{-15pt}
\end{table}

\subsection{Dense Reconstruction and Depth}
\label{sec:exp_recon}

Tab.~\ref{tab:recon} and Tab.~\ref{tab:kitti_depth} report reconstruction and depth accuracy. Note that 7Scenes is part of our training data. HorizonStream achieves the best online reconstruction quality across four benchmarks, mainly due to more accurate pose estimation. On 7Scenes, several baselines have inflated mean CD due to large errors on Chess, Pumpkin, and RedKitchen. On KITTI depth, HorizonStream approaches the best offline methods among the compared baselines.

\begin{wraptable}{R}{0.4\columnwidth}
\vspace{-12pt}
\centering
\caption{ATE ($\downarrow$) ablation on vKITTI2.}
\label{tab:ablation}
\setlength{\tabcolsep}{4pt}
\renewcommand{\arraystretch}{1.08}
\small
\resizebox{\linewidth}{!}{
\begin{tabular}{lccc}
\toprule
Variant & 80f & 200f & 400f \\
\midrule
\textbf{Full model} & \textbf{0.42} & \textbf{0.71} & \textbf{1.20} \\
\midrule
\multicolumn{4}{l}{\textit{Geometric Linear Attention}} \\
\quad w/o Geometric Linear Attention & 0.83 & 2.06 & 5.38 \\
\quad w/o channel-wise gating & 0.67 & 1.43  & 3.21 \\
\quad replace with TTT-like fast weight & 0.58 & 1.56 & 3.96 \\
\midrule
\multicolumn{4}{l}{\textit{Gated Window Attention}} \\
\quad w/o Gated Window Attention & 0.78 & 2.64 & 7.46 \\
\quad w/o head-wise output gating & 0.61 & 1.74 & 4.06 \\
\quad w/o Geometric RoPE, 2D spatial only & 0.64 & 1.22 & 2.58 \\
\midrule
\multicolumn{4}{l}{\textit{Scale and pose}} \\
\quad w/o MRT & 0.55 & 1.32 & 3.34 \\
\quad single-token pose, no aggregation & 0.51 & 1.10 & 2.67 \\
\bottomrule
\end{tabular}
}
\vspace{-8pt}
\end{wraptable}

\noindent \textbf{Ablation study.}
\textit{Geometric Linear Attention.}
Removing it entirely causes severe drift, confirming the necessity of long-term state.
Disabling channel-wise gating or replacing it with TTT-like fast weights both degrade performance, especially at longer horizons, showing that per-channel bounded retention is critical.
Fig.~\ref{fig:gamma_1} visualizes the learned effective lifetimes $\tau = -1/\log\bar{\gamma}$, which form a continuous spectrum across channels and layers.
Fig.~\ref{fig:ab_band} further shows that replacing this learned spectrum with any fixed band degrades accuracy, confirming the necessity of multi-timescale retention.

\textit{Gated Window Attention.}
Removing it yields the severe degradation, reflecting the importance of fine-grained spatial matching within each window.
%todo add some discussion about Geometric Linear Attention help scale stabel
Fig.~\ref{ab:ate} shows that head-wise output gating and Spatiotemporal RoPE are complementary: removing either substantially increases drift over long sequences.

\textit{Scale and pose readout.}
Metric Readout Tokens and multi-token pose aggregation each contribute consistent gains.
Additional results on loop closure, memory and runtime scaling, and training convergence are in Appendix~\ref{app:results}.

\begin{figure}[t!]
  \centering
  
  % 第一张图 (a)
  \begin{subfigure}[b]{0.32\textwidth}
    \centering
    \includegraphics[width=\linewidth]{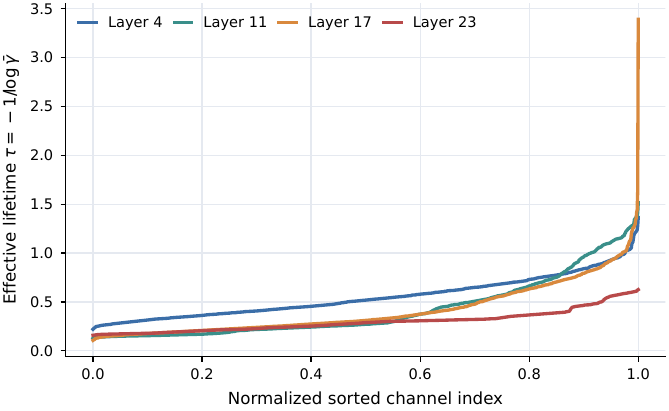}
    \caption{}
    \label{fig:gamma_1}
  \end{subfigure}%
  \hfill
  % 第二张图 (b)
  \begin{subfigure}[b]{0.32\textwidth}
    \centering
    \includegraphics[width=\linewidth]{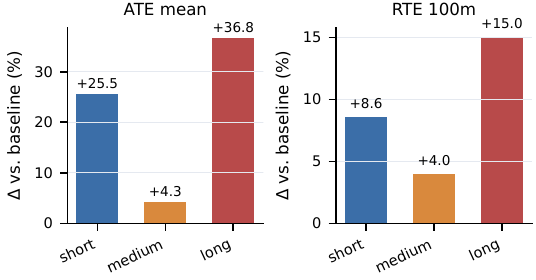}
    \caption{}
    \label{fig:ab_band}
  \end{subfigure}%
  \hfill
  % 第三张图 (c)
  \begin{subfigure}[b]{0.32\textwidth}
    \centering
    \includegraphics[width=\linewidth]{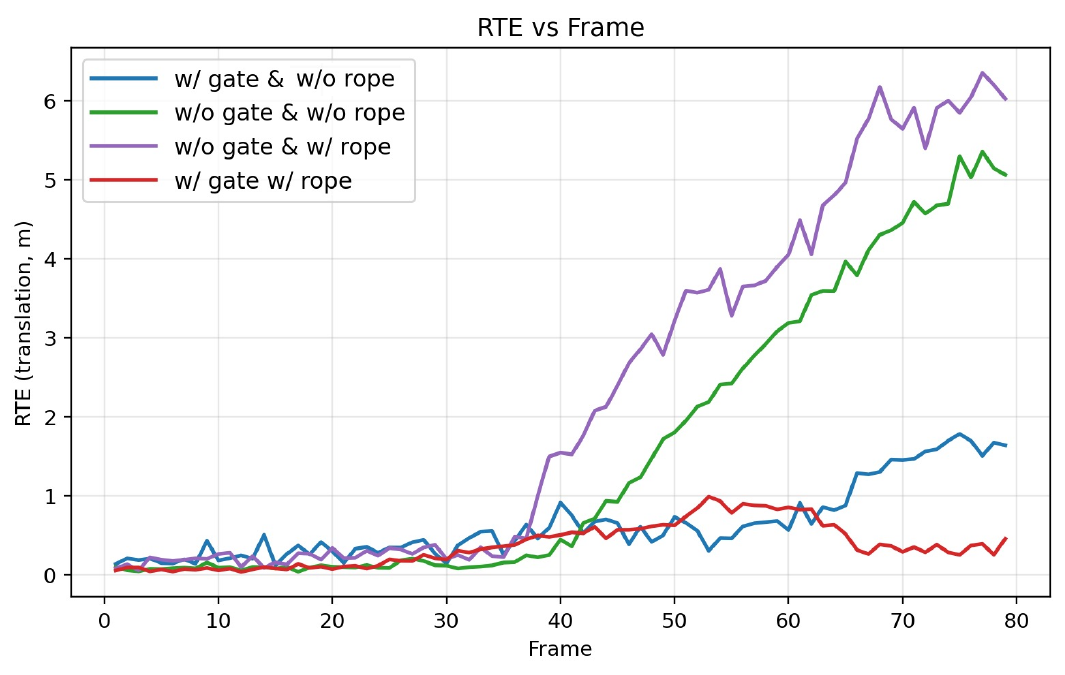}
    \caption{}
    \label{ab:ate}
  \end{subfigure}
  \vspace{-7pt}
  \caption{
\textbf{(a) Learned retention spectra in Geometric Linear Attention.}
Effective lifetimes \(\tau=-1/\log\bar{\gamma}\) vary across channels and layers. Layer~4 exhibits broad mid-range retention, while Layer~11 develops a pronounced long-retention tail, supporting channel-wise multi-timescale propagation.
\textbf{(b) Retention-band ablation.}
Replacing the learned channel-wise retention spectrum with fixed short-, medium-, or long-horizon bands increases trajectory error, showing that stable long-sequence propagation requires learned multi-timescale retention.
\textbf{(c) Long-sequence stability of Gated Window Attention.}
Head-wise gating and 3D RoPE are complementary: removing either causes error growth over time, while using both keeps the model stable.
  }
  \label{fig:main_combined_figure}
  \vspace{-15pt}
\end{figure}

% ====================================================================
% 5. DISCUSSION
% ====================================================================
\noindent \textbf{Discussion.}
\label{sec:discussion}
HorizonStream predicts poses using a local window of only 10 frames, suggesting that compact local geometric evidence is sufficient for accurate pose estimation while reducing memory cost and improving inference speed. A larger pose window may further improve the model's internal loop-closure ability. Additionally, for extremely long sequences with repeated revisits, the fixed-size recurrent state still misses fine-grained details, as shown in Fig.~\ref{ab:failure}. Dynamic foreground objects can also corrupt local geometric evidence in the input video. The optional loop-closure module is currently parameterized separately, and its optimization settings could be further refined. The method also relies on strong pretraining and substantial training compute, can be less reliable in texture-poor scenes, and may propagate pose or reconstruction errors to downstream robotic or autonomous-driving systems.
% ====================================================================
% 6. CONCLUSION
% ====================================================================
\section{Conclusion}
\label{sec:conclusion}

We presented HorizonStream, a streaming 3D reconstruction framework built on an evidence influence kernel that unifies long-term temporal memory and short-term spatial matching. Its channel-wise retention, Gated Window Attention, and Metric Readout Tokens jointly support multi-timescale evidence propagation, reliable local matching, and consistent metric readout. Trained on 48-frame clips, it generalizes to sequences exceeding 10{,}000 frames with constant memory and linear time. Experiments across diverse indoor and outdoor datasets demonstrate strong trajectory accuracy and coherent dense reconstruction, highlighting the importance of learned evidence lifetimes for stable long-horizon streaming.

\begin{ack}
This work is supported by the National Natural Science Foundation of China (No. 62406267), Guangdong Provincial Project (No. 2024QN11X072), Guangzhou-HKUST(GZ) Joint Funding Program (No. 2025A03J3956) and Guangzhou Municipal Education Project (No. 2024312122).
\end{ack}

\bibliographystyle{plainnat}
\bibliography{bibliography}

\clearpage
% ====================================================================
% APPENDIX
% ====================================================================
% ====================================================================
% APPENDIX
% ====================================================================
\appendix

\section{Geometric Attention Dilution}
\label{app:softmax_dilution}

We formalize why causal softmax attention cannot serve as long-range cross-window memory in streaming 3D reconstruction.

Let $\Omega_t$ denote the set of 3D points visible at time $t$, and define the \emph{geometric relevance} of frame $i$ to frame $t$ as the co-visibility ratio $r(i,t) = |\Omega_i \cap \Omega_t|/|\Omega_t|$.
For a camera exploring new regions, the co-visibility set $\mathcal{R}_t = \{i \leq t : r(i,t) > 0\}$ is bounded by $|\mathcal{R}_t| \leq W_{\text{geo}}$, determined by scene geometry and camera speed.

\begin{proposition}[Geometric Attention Dilution]
\label{prop:dilution}
Let $\alpha_i = \mathrm{softmax}(\bq_t^\top \bk_i / \sqrt{d})_{i=1}^t$ be causal softmax attention weights with bounded scores $|\bq_t^\top \bk_i / \sqrt{d}| \leq M$.
The total attention on geometrically relevant frames satisfies:
\begin{equation}
\label{eq:dilution}
\sum_{i \in \mathcal{R}_t} \alpha_i \leq \frac{1}{1 + \dfrac{t - W_{\mathrm{geo}}}{W_{\mathrm{geo}}} \cdot e^{-2M}}.
\end{equation}
For $t > W_{\mathrm{geo}}(1 + e^{2M})$, more than half the attention mass falls on geometrically irrelevant frames.
Even under perfect score discrimination, the relevant fraction decays as $O(W_{\mathrm{geo}} e^{2M} / t)$.
\end{proposition}

\begin{proof}
Assign the best-case scores: $+M$ to all $W_{\text{geo}}$ relevant frames, $-M$ to all others.
Then:
\[
\sum_{i \in \mathcal{R}_t} \alpha_i \leq \frac{W_{\text{geo}} \cdot e^M}{W_{\text{geo}} \cdot e^M + (t - W_{\text{geo}}) \cdot e^{-M}}.
\]
Dividing numerator and denominator by $W_{\text{geo}} \cdot e^M$ yields~\eqref{eq:dilution}.
Any suboptimal score assignment only worsens the bound.
For $t \gg W_{\text{geo}} e^{2M}$, the bound is $O(W_{\text{geo}} e^{2M} / t) \to 0$.
\end{proof}

\begin{remark}
In practice, models mitigate the wasted attention by concentrating irrelevant mass onto \emph{sink tokens}~\citep{xiao2024streamingllm,gu2025sinkemerge} whose values collapse toward zero.
This symptom does not resolve the underlying problem: the geometrically useful signal fraction still vanishes as $O(1/t)$, and the $O(t)$ KV-cache cost remains.
Both properties rule out causal softmax as a cross-window memory mechanism.
\end{remark}

\section{Extended Theoretical Analysis}
\label{app:theory}

\subsection{Zero-Forgetting Contamination and Stability}

We state and prove the two core propositions motivating selective forgetting in the recurrent geometric state.

\begin{proposition}[Zero-Forgetting Contamination]
\label{prop:persistence}
Under zero forgetting ($\gamma \equiv 1$), the initial state $\bS_0$ contributes to every output with undiminished magnitude:
\[
\bo_t = \bq_t^\top \bS_0 + \sum_{i=1}^t \bq_t^\top \bk_i \tilde{\bv}_i^\top.
\]
No amount of new evidence can dilute the initial state.
\end{proposition}

\begin{proof}
Under the ungated update $\bS_t = \bS_{t-1} + \bk_t\tilde{\bv}_t^\top$, unrolling gives $\bS_t = \bS_0 + \sum_{i=1}^t \bk_i\tilde{\bv}_i^\top$.
Hence $\bo_t = \bq_t^\top\bS_t = \bq_t^\top\bS_0 + \sum_{i=1}^t \bq_t^\top\bk_i\tilde{\bv}_i^\top$.
The $\bq_t^\top\bS_0$ term is independent of $t$ and never diminishes.
\end{proof}

\begin{remark}
This is the root cause of the degradation observed in TTT without reset~\citep{chen2025ttt3r,zhang2026loger}, CUT3R~\citep{wang2025cut3r}, and standard linear attention when applied to long streaming sequences: the state is permanently anchored to initialization regardless of camera motion.
\end{remark}

\begin{proposition}[Bounded Initial-State Influence]
\label{prop:bounded_initial_state}
Under the channel-wise retention update~\eqref{eq:gate_update}, if 
\(\bar{\gamma} = \sup_{t,c} |\gamma_t^{(c)}| < 1\), 
the contribution of the initial state decays exponentially:
\[
\left\|
\bq_t^\top 
\left(\prod_{j=1}^{t} \operatorname{diag}(\bgamma_j)\right) 
\bS_0
\right\|
\leq 
\|\bq_t\| \cdot \|\bS_0\|_F \cdot \bar{\gamma}^{\,t}
\to 0
\quad \text{as } t \to \infty .
\]
\end{proposition}

\begin{proof}
Unrolling the gated recurrence gives
\[
\bS_t =
\left(\prod_{j=1}^{t} \operatorname{diag}(\bgamma_j)\right)\bS_0
+
\sum_{i=1}^{t}
\left(\prod_{j=i+1}^{t} \operatorname{diag}(\bgamma_j)\right)
\bk_i \tilde{\bv}_i^\top .
\]
Since 
\(\|\operatorname{diag}(\bgamma_j)\|_{\mathrm{op}} \leq \bar{\gamma}\),
submultiplicativity gives
\[
\left\|
\prod_{j=1}^{t} \operatorname{diag}(\bgamma_j)
\right\|_{\mathrm{op}}
\leq
\bar{\gamma}^{\,t}.
\]
Applying Cauchy--Schwarz yields the stated bound.
\end{proof}

\noindent
Thus, channel-wise retention with \(\bar{\gamma}<1\) is a sufficient condition for closing the influence of the initial state. The recurrent state remains adaptive to incoming evidence rather than being anchored to its initialization.
% ====================================================================
\subsection{Effective Memory Horizon}

\begin{proposition}[Per-Channel Memory Horizon]
\label{prop:horizon}
For a channel $c$ with constant gate $\gamma^{(c)} \in (0,1)$, define the effective memory horizon as $\tau^{(c)} = -1/\log \gamma^{(c)}$.
Then the contribution of observation $i$ to the output at time $t$ decays as:
\[
w^{(c)}(t, i) = (\gamma^{(c)})^{t-i} = e^{-(t-i)/\tau^{(c)}}.
\]
For $t - i > 3\tau^{(c)}$, the contribution is below $5\%$ of its original weight.
\end{proposition}

\begin{proof}
Direct computation: $w^{(c)}(t,i) = (\gamma^{(c)})^{t-i} = e^{(t-i)\log\gamma^{(c)}} = e^{-(t-i)/\tau^{(c)}}$.
At $t - i = 3\tau^{(c)}$: $w = e^{-3} \approx 0.050$.
\end{proof}

\begin{corollary}[Heterogeneous Memory Partitioning]
\label{cor:partition}
With channel-wise gating, the state $\bS_t \in \R^{d \times d}$ is implicitly partitioned into subspaces with different memory lifetimes.
Let $\mathcal{C}_{fast} = \{c : \gamma^{(c)} < \gamma_{th}\}$ and $\mathcal{C}_{slow} = \{c : \gamma^{(c)} \geq \gamma_{th}\}$ for some threshold $\gamma_{th}$.
Then the fast subspace $\bS_t[\mathcal{C}_{fast}, :]$ acts as a short-term feature buffer with horizon $\tau_{fast} \ll T$, while the slow subspace $\bS_t[\mathcal{C}_{slow}, :]$ acts as a long-term geometric memory with horizon $\tau_{slow} \gg W$.
This partitioning is learned end-to-end and adapts to the geometric content of the training data.
\end{corollary}

\subsection{State Norm Boundedness}

A key practical concern is whether the persistent state $\bS_t$ remains bounded as $t \to \infty$.

\begin{proposition}[Bounded State Norm]
\label{prop:bounded}
Under the gated update $\bS_t = \mathrm{diag}(\bgamma_t)\bS_{t-1} + \bk_t\tilde{\bv}_t^\top$ with $\bar{\gamma} = \sup_{t,c}|\gamma_t^{(c)}| < 1$ and bounded inputs $\|\bk_t\| \leq B_k$, $\|\tilde{\bv}_t\| \leq B_v$ for all $t$, the Frobenius norm of the state is uniformly bounded:
\[
\|\bS_t\|_F \leq \bar{\gamma}^t \|\bS_0\|_F + \frac{B_k B_v}{1 - \bar{\gamma}}.
\]
In particular, $\limsup_{t\to\infty} \|\bS_t\|_F \leq B_k B_v / (1 - \bar{\gamma})$.
\end{proposition}

\begin{proof}
By submultiplicativity and the triangle inequality:
\begin{align*}
\|\bS_t\|_F &\leq \|\mathrm{diag}(\bgamma_t)\|_{\text{op}} \|\bS_{t-1}\|_F + \|\bk_t\tilde{\bv}_t^\top\|_F \\
&\leq \bar{\gamma}\|\bS_{t-1}\|_F + B_k B_v.
\end{align*}
Unrolling the recurrence:
$\|\bS_t\|_F \leq \bar{\gamma}^t\|\bS_0\|_F + B_k B_v \sum_{i=0}^{t-1}\bar{\gamma}^i \leq \bar{\gamma}^t\|\bS_0\|_F + B_k B_v/(1 - \bar{\gamma})$.
\end{proof}

\begin{remark}
The bound in Proposition~\ref{prop:bounded} guarantees numerical stability without periodic state resets.
In contrast, ungated linear attention ($\bar{\gamma} = 1$) yields $\|\bS_t\|_F \leq \|\bS_0\|_F + t \cdot B_k B_v$, which grows linearly and eventually requires resets to prevent overflow, as observed in TTT-based methods~\citep{zhang2026loger}.
\end{remark}

\subsection{Formal Connection to Test-Time Training}

We formalize the relationship between Geometric Linear Attention and TTT.

\begin{proposition}[Geometric Linear Attention as Discounted TTT]
\label{prop:ttt}
Consider a linear model $f_{\bS}(\bk) = \bS^\top\bk$ trained online to minimize $\ell_t = \|\bS^\top\bk_t - \bv_t\|^2$ with the discounted objective~\eqref{eq:objective}.
One step of gradient descent at learning rate $\eta$ on the discounted objective, starting from the previous iterate $\bS_{t-1}$ discounted by $\gamma_t$, produces:
\[
\bS_t = \gamma_t \bS_{t-1} - \frac{\eta}{2}\nabla_\bS\ell_t\big|_{\bS=\gamma_t\bS_{t-1}} = \gamma_t \bS_{t-1} + \eta\bk_t(\bv_t - \gamma_t\bS_{t-1}^\top\bk_t)^\top.
\]
When $\gamma_t \equiv 1$, this reduces to the standard online linear regression update, which \citet{liu2026tttla} showed is equivalent to linear attention.
The gated form thus extends the TTT--linear-attention equivalence to the discounted setting: \emph{Geometric Linear Attention is equivalent to discounted test-time training}.
\end{proposition}

\begin{proof}
The gradient of $\ell_t$ at $\bS' = \gamma_t\bS_{t-1}$ is $\nabla_\bS\ell_t|_{\bS'} = 2\bk_t(\bS'^\top\bk_t - \bv_t)^\top$.
Gradient descent: $\bS_t = \bS' - (\eta/2)\nabla\ell_t|_{\bS'} = \gamma_t\bS_{t-1} + \eta\bk_t(\bv_t - \gamma_t\bS_{t-1}^\top\bk_t)^\top$.
Setting $\gamma_t = 1$ recovers $\bS_t = \bS_{t-1} + \eta\bk_t(\bv_t - \bS_{t-1}^\top\bk_t)^\top$, the undiscounted TTT/linear-attention update of \citet{liu2026tttla}.
\end{proof}

\section{Implementation Details}
\label{app:arch}

\subsection{Architecture Details}

The backbone consists of 24 transformer layers alternating between frame blocks and global blocks. Geometric Linear Attention layers are placed at layers 4, 11, 17, and 23. Each Geometric Linear Attention layer reads and updates the persistent state after Gated Window Attention operates. The head-wise gate bias is initialized to 2.0 to preserve pretrained attention at the start of training. Geometric Linear Attention gates are initialized with high bias to produce $\gamma \approx 1$, gradually learning channel-wise retention as training progresses.

The pose consensus head uses a lightweight transformer with residual corrections over $K$ rounds. Each round refines the translation, rotation quaternion, and focal length. The depth head uses DPT-style multi-scale fusion from four intermediate layers.

\subsection{Training Hyperparameters}
\label{app:hyperparams}

\noindent \textbf{Input and model dimensions.}
Input images are resized to $518{\times}518$. The Geometric Linear Attention state has dimension $\bS \in \mathbb{R}^{d \times d}$ with $d{=}1024$.

Scale loss is applied only on metric-scale samples. Depth loss uses SmoothL1 with confidence weighting. We apply random color jitter, random cropping, and random horizontal flip.

\noindent \textbf{Local-window sampling.}
During training, we randomly vary the local-window size while keeping each training sample at 48 frames. The recurrent state is propagated across computational chunks. At inference, we use a local window of 10 frames as an accuracy--efficiency trade-off.

\subsection{Training Data}
\label{app:training_data}

We train on 24 datasets covering indoor, outdoor, driving, and synthetic environments. Video data is sampled with variable temporal stride from 1 to 8. Unordered image collections are converted to pseudo-temporal sequences via camera graph traversal. Tab.~\ref{tab:training_data_ratio} lists per-dataset sampling ratios.

\begin{table}[H]
  \centering
  \caption{Training data composition. We use dataset-specific sampling ratios in two stages. Stage~1 emphasizes data diversity, while Stage~2 increases the proportion of metric-scale datasets.}
  \label{tab:training_data_ratio}
  \setlength{\tabcolsep}{5pt}
  \renewcommand{\arraystretch}{1.08}
  \resizebox{0.6\linewidth}{!}{
  \begin{tabular}{lccc}
  \toprule
  Dataset & Stage 1 Ratio & Stage 2 Ratio & Metric Scale \\
  \midrule
  blendedmvs/train & 4.90\% & 2.40\% & \xmark \\
  megadepth & 1.00\% & -- & \xmark \\
  hypersim/train & 7.30\% & 5.10\% & \cmark \\
  hypersim/val & 1.50\% & 1.40\% & \cmark \\
  ase & 9.80\% & 9.50\% & \cmark \\
  scannetpp & 7.30\% & 6.10\% & \cmark \\
  tartanair & 7.30\% & 6.10\% & \cmark \\
  vkitti2 & 9.80\% & 9.50\% & \cmark \\
  mapillary & 9.80\% & 9.50\% & \cmark \\
  waymo & 9.80\% & 9.50\% & \cmark \\
  wildrgbd/train & 7.30\% & 7.10\% & \cmark \\
  co3dv2/train & 7.00\% & -- & \xmark \\
  dl3dv & 9.80\% & 9.50\% & \xmark \\
  mapfree/train & 4.90\% & 4.70\% & \cmark \\
  replica\_niceslam & 0.50\% & 0.86\% & \xmark \\
  7scenes & 0.50\% & 0.86\% & \xmark \\
  GTAV\_1080 & 1.00\% & 0.90\% & \xmark \\
  spring & 0.50\% & -- & \xmark \\
  point\_odyssey/train & -- & 0.86\% & \xmark \\
  point\_odyssey/val & -- & 0.86\% & \xmark \\
  ARKitscenes & -- & 0.86\% & \cmark \\
  unrealstereo4k & -- & 0.90\% & \xmark \\
  OmniWorld & -- & 4.70\% & \cmark \\
  matrixcity\_d2 aerial & -- & 2.40\% & \cmark \\
  matrixcity\_d2 street & -- & 2.40\% & \cmark \\
  Internal Long-Sequence Data  & -- & 4.00\% & \cmark \\
  \bottomrule
  \end{tabular}
  }
\end{table}

\begin{figure}[t]
  \centering
  \includegraphics[width=0.5\linewidth]{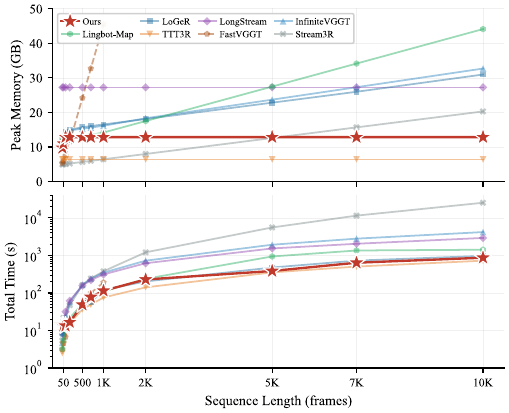}
  \vspace{-8pt}
  \caption{GPU memory and runtime scaling. HorizonStream keeps peak GPU memory nearly constant and scales smoothly to $10\mathrm{K}$ frames, while competing methods require increasing GPU memory or higher runtime on long sequences.}
  \label{ab:time_memory}
\end{figure}

\section{Evaluation Dataset Details}
\label{app:eval_data}

We evaluate all sequences at full length without frame subsampling. Below we describe per-dataset evaluation protocols.

\noindent \textbf{Failure penalty.}
For trajectory ATE tables, ``--'' denotes a missing or failed result. When computing averages, each missing entry is assigned the largest reported error in the same sequence column as a failure penalty.

\noindent \textbf{KITTI.} All 11 sequences (00--10) are evaluated with full frames.

\noindent \textbf{vKITTI2.} We evaluate all morning-condition scenes across the five virtual environments (Scene01, Scene02, Scene06, Scene18, Scene20). This is a training-overlap evaluation.

\noindent \textbf{7Scenes.} We evaluate sequence~01 from each of the seven scenes. Among these evaluated sequences, those from Chess, Fire, Office, and RedKitchen are included in our training data, whereas those from Heads, Pumpkin, and Stairs are excluded.

\noindent \textbf{Waymo Open.} We select 9 segments not present in our training set:
163453191 (198 frames, 160\,m),
183829460 (199 frames, 42\,m),
315615587 (199 frames, 165\,m),
346181117 (199 frames, 351\,m),
371159869 (196 frames, 273\,m),
405841035 (199 frames, 86\,m),
460417311 (198 frames, 266\,m),
520018670 (199 frames, 135\,m),
610454533 (198 frames, 63\,m).
Although Waymo is part of our training data, these specific segments are held out to evaluate generalization on unseen driving scenes.

\noindent \textbf{ScanNet++.} We evaluate on 5 scenes: 419cbe7c11, 98b4ec142f, bb87c292ad, c24f94007b, ebc200e928. The evaluation segments from these five scenes are excluded from training.

\noindent \textbf{Oxford Spires.} We evaluate on 12 sequences. Since the ground-truth point clouds and images are in different coordinate systems, we perform image-to-ground-truth point cloud alignment. The number of aligned images varies across subsets, increasing evaluation difficulty. Per-sequence frame counts and trajectory lengths are shown in Tab.~\ref{tab:oxford_spires_ate}.

\begin{figure}[t]
  \centering
  \includegraphics[width=0.42\columnwidth]{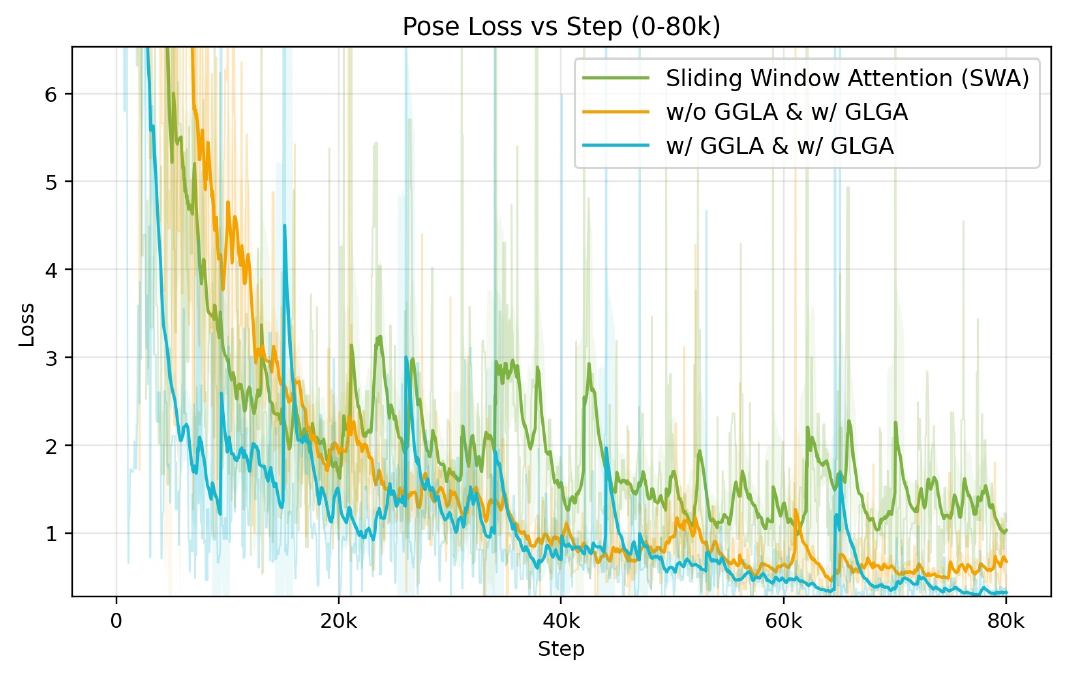}
  \vspace{-8pt}
  \caption{Training convergence under different attention mechanisms. In the legend, GGLA and GLGA denote Geometric Linear Attention (GLA) and Gated Window Attention (GWA), respectively; SWA denotes the sliding-window attention baseline. GLA with channel-wise gating converges faster and reaches a lower final loss.}
  \label{ab:loss}
  \vspace{-8pt}
\end{figure}

\noindent \textbf{VBR.} Following the LoGeR~\citep{zhang2026loger} setting, all 7 sequences are evaluated at full length (8{,}815 to 18{,}846 frames, up to 5.2\,km).

\begin{figure}[t]
  \centering
  \includegraphics[width=0.5\linewidth]{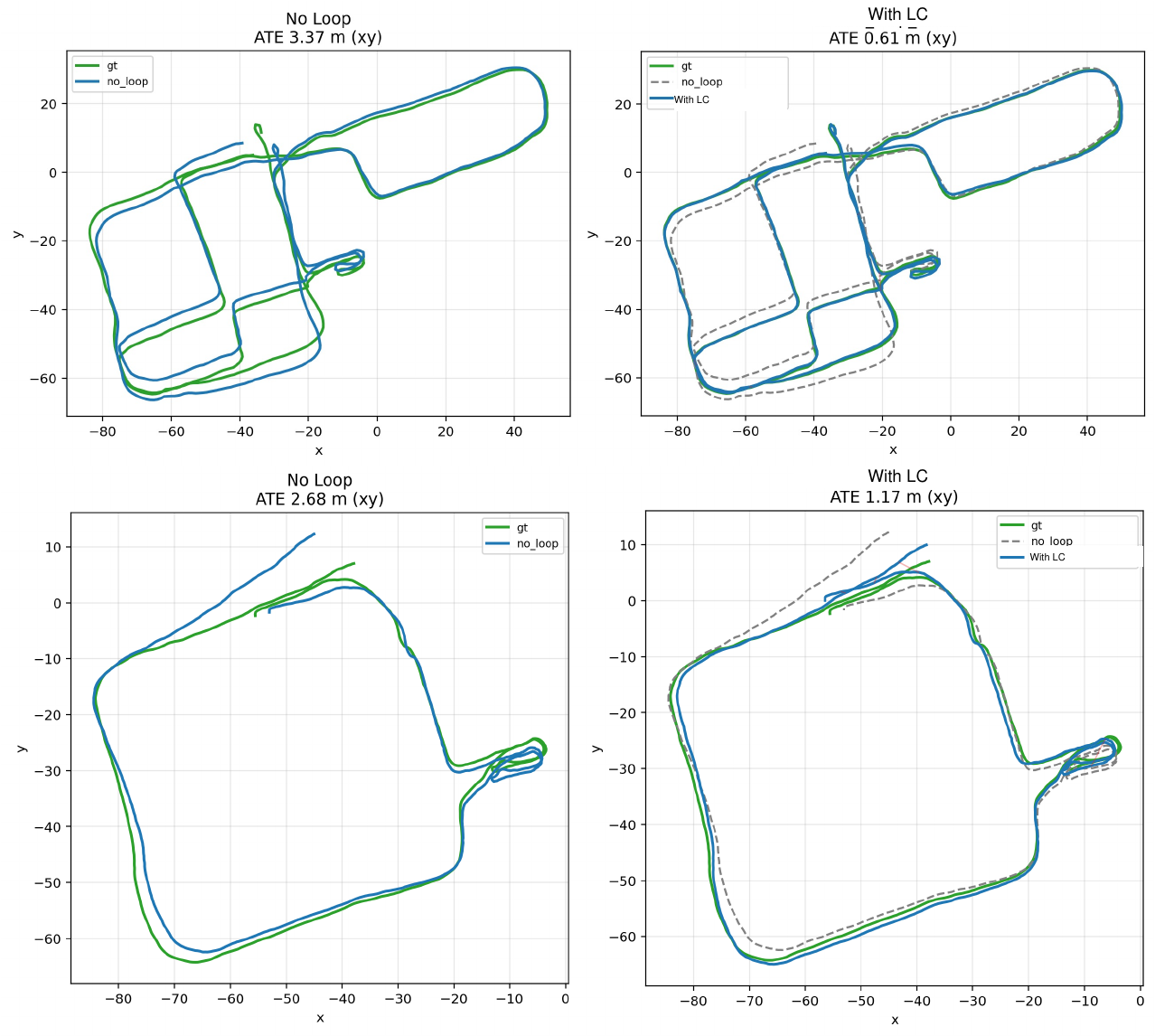}
  \vspace{-8pt}
  \caption{Effect of loop closure on long sequences. Loop closure reduces ATE on sequences with revisited regions while maintaining performance elsewhere.}
  \label{ab:lc}
\end{figure}

\noindent \textbf{TUM RGB-D and ETH3D.} Standard evaluation protocols with full sequences.

\section{Additional Experimental Results}
\label{app:results}

Tab.~\ref{tab:oxford_spires_ate} reports per-sequence ATE on the 12 Oxford Spires evaluation sites. HorizonStream achieves the lowest average ATE among all online methods, with particularly large margins on long-trajectory sequences such as college5 and christ2.

\begin{table}[H]
\centering
\caption{Quantitative comparison on Oxford Spires. We report ATE, where lower is better.}
\label{tab:oxford_spires_ate}
\setlength{\tabcolsep}{2.3pt}
\renewcommand{\arraystretch}{1.08}
\newcommand{\seqinfo}[2]{%
{\scriptsize\textcolor{cyan!70!black}{#1 fr.}}\\[-1pt]
{\scriptsize\textcolor{red!70!black}{#2 m}}%
}
\resizebox{\textwidth}{!}{
\begin{tabular}{c l ccccccccccccc}
\toprule
& \multirow{2}{*}{Method}
& \multicolumn{12}{c}{Oxford Spires ATE $\downarrow$}
& \multirow{2}{*}{Avg.} \\
\cmidrule(lr){3-14}
&
& \makecell{college2\\\seqinfo{787}{290}}
& \makecell{college3\\\seqinfo{757}{280}}
& \makecell{college4\\\seqinfo{701}{773}}
& \makecell{college5\\\seqinfo{636}{696}}
& \makecell{observ1\\\seqinfo{353}{393}}
& \makecell{observ2\\\seqinfo{351}{387}}
& \makecell{blenheim1\\\seqinfo{57}{341}}
& \makecell{blenheim2\\\seqinfo{25}{316}}
& \makecell{blenheim5\\\seqinfo{12}{259}}
& \makecell{christ2\\\seqinfo{567}{629}}
& \makecell{christ3\\\seqinfo{289}{309}}
& \makecell{bodleian2\\\seqinfo{22}{537}}
& \\
\midrule

\multirow{6}{*}{\rotatebox[origin=c]{90}{Opt.-based}}
& MASt3R-SLAM & 15.97 & 31.89 & -- & -- & 20.05 & 21.44 & -- & 45.48 & 50.62 & -- & -- & -- & 37.73 \\
& VGGT-SLAM & -- & 13.53 & 14.64 & -- & 29.12 & -- & 40.04 & 19.20 & 10.36 & -- & -- & 80.54 & 31.00 \\
& COLMAP & 0.06 & 0.05 & 0.05 & 0.32 & 0.17 & 0.26 & 0.21 & 0.06 & -- & 39.55 & 14.99 & -- & 15.57 \\
& MASt3R-SfM & 21.55 & 13.57 & 32.53 & 36.84 & 23.14 & 27.35 & 25.84 & 37.22 & 47.80 & 35.55 & 14.71 & 69.48 & 32.13 \\
& DPVO & 14.31 & 31.60 & 39.26 & 34.79 & 26.91 & 28.14 & 35.37 & 44.76 & 47.96 & 19.98 & 16.00 & 69.31 & 34.03 \\
& DROID-SLAM & 20.96 & 23.09 & 20.39 & 38.30 & 17.20 & 23.86 & 30.39 & 46.78 & 47.08 & 16.97 & 16.04 & 71.85 & 31.08 \\

\midrule

\multirow{4}{*}{\rotatebox[origin=c]{90}{Offline}}
& VGGT-Long & 6.55 & 14.32 & 13.20 & 40.27 & 11.95 & 6.47 & 24.77 & 19.20 & 10.36 & 19.40 & 15.75 & 80.54 & 21.90 \\
& FastVGGT & 24.10 & 32.54 & 39.93 & 37.07 & 26.86 & 26.57 & 33.04 & -- & 38.14 & 40.90 & 15.79 & 78.52 & 36.58 \\
& LoGeR & 6.76 & 9.37 & 5.46 & 7.66 & 6.16 & 5.62 & 26.82 & 26.90 & 32.92 & 19.33 & 3.91 & 73.46 & 18.70 \\
& LoGeR$^{*}$ & 2.79 & 11.92 & 3.97 & 9.05 & 4.78 & 4.71 & 27.34 & 20.61 & 21.80 & 8.67 & 2.86 & 71.04 & 15.79 \\
\midrule

\multirow{8}{*}{\rotatebox[origin=c]{90}{Online Fwd.}}
& CUT3R & 30.68 & 23.73 & 31.31 & 30.80 & 25.32 & 26.21 & 31.15 & 37.09 & 38.71 & 37.25 & 14.33 & 62.68 & 32.44 \\
& TTT3R & 27.75 & 23.39 & 40.99 & 42.39 & 28.17 & 26.73 & 37.07 & 44.93 & 38.34 & 36.69 & 14.92 & 73.21 & 36.21 \\
& STream3R & 33.08 & 32.68 & 43.54 & 41.82 & 28.40 & 28.37 & 36.21 & 41.72 & 31.96 & 44.98 & 15.47 & 72.60 & 37.57 \\
& StreamVGGT & 31.97 & 31.09 & 43.55 & 42.15 & 29.09 & 28.24 & 36.08 & 42.57 & 30.16 & 41.23 & 15.74 & 75.18 & 37.25 \\
& InfiniteVGGT & 25.71 & 27.24 & 27.94 & 28.33 & 25.81 & 24.04 & 35.69 & 44.49 & 19.33 & 38.59 & 12.83 & 71.81 & 31.82 \\
& Lingbot-Map & 2.17 & 3.61 & 19.99 & 12.01 & 9.99 & 6.23 & 8.23 & 14.59 & 39.67 & 15.79 & 12.86 & 40.38 & 15.46 \\
& Ours & 2.81 & 0.76 & 10.87 & 2.84 & 1.43 & 2.17 & 6.71 & 11.62 & 33.53 & 4.70 & 3.53 & 31.59 & 9.38 \\

& Ours w/ LC & 1.05 & 0.70 & 9.85 & 0.88 & 1.08 & 1.16 & 6.71 & 11.62 & 33.53 & 4.70 & 1.68 & 31.59 & 8.71 \\
\bottomrule
\end{tabular}
}
\end{table}

\noindent\textbf{Loop closure.}
Fig.~\ref{ab:lc} shows the effect of the optional loop-closure module on long sequences.
Loop closure reduces ATE on sequences with revisited regions while maintaining comparable performance.

\noindent\textbf{Memory and runtime scaling.}
Fig.~\ref{ab:time_memory} reports peak GPU memory and wall-clock time as sequence length grows from 200 to 10{,}000 frames.
HorizonStream maintains nearly constant peak memory and scales smoothly, while competing methods either run out of memory or exhibit super-linear runtime growth.

\noindent\textbf{Training convergence.}
Fig.~\ref{ab:loss} compares training loss curves when using different attention mechanisms for cross-window propagation.
Geometric Linear Attention with channel-wise gating converges faster and reaches a lower final loss than the ungated and softmax-attention variants, reflecting the benefit of bounded multi-timescale retention during training.

\noindent\textbf{Failure cases.}
Fig.~\ref{ab:failure} shows representative failure cases on ultra-long sequences.
Errors mainly occur in sequences with dense revisits or visually ambiguous regions, where the fixed-size recurrent state does not preserve sufficient fine-grained information for precise relocalization.
The optional loop-closure module partially mitigates these failures.

\begin{figure}[t]
  \centering
  \includegraphics[width=\columnwidth]{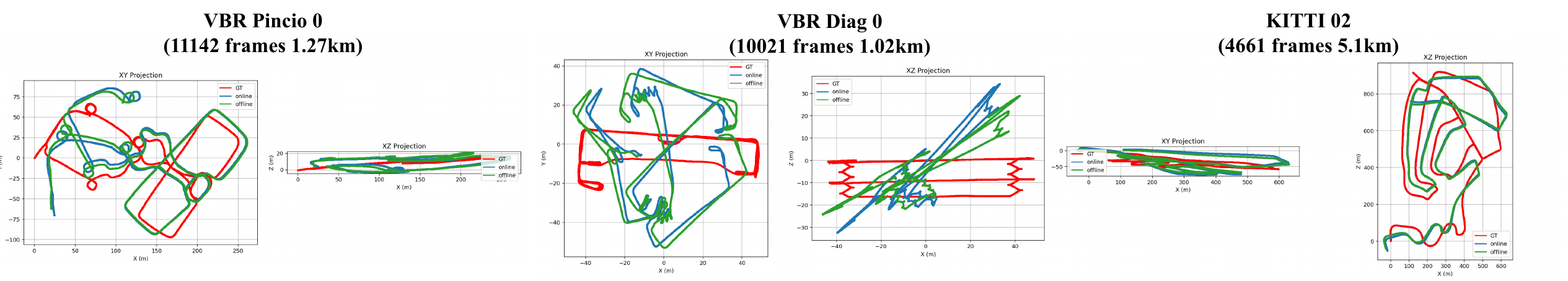}
  \vspace{-8pt}
  \caption{Failure cases on ultra-long sequences. Ground-truth trajectory (red), online prediction (blue), and loop-closure refined trajectory (green). Failures occur mainly in sequences with dense revisits or visually ambiguous regions.}
  \label{ab:failure}
  \vspace{-8pt}
\end{figure}

\clearpage
\subsection{Channel-to-Geometry Linear Probing}
\label{app:probe}

\noindent \textbf{Linear probing of recurrent geometric states.}
We further examine whether frozen Geometric Linear Attention states contain linearly decodable geometric information.
For each chunk, we extract 1024-dimensional state features from all four Geometric Linear Attention layers.
We train ridge regressors on KITTI sequences \texttt{00} and \texttt{02}, and evaluate on the held-out sequence \texttt{05}.
The probes predict four geometric error targets: local translation error, long-range scale log error, long-range translation error, and segment scale log error.

Fig.~\ref{fig:probe_performance} shows that segment-scale log error is the most reliably predictable target from frozen Geometric Linear Attention states, suggesting that the recurrent state contains measurable metric-related information.

Fig.~\ref{fig:probe_bands} analyzes where the segment-scale signal resides across retention bands.
The probe weights are distributed across short-, medium-, and long-retention channels, rather than being concentrated in a single band.
This supports the view that metric evidence is represented across the learned retention spectrum and benefits from multi-timescale propagation.

\begin{figure}[!htbp]
  \centering
  \begin{minipage}[t]{0.57\textwidth}
    \centering
    \includegraphics[width=\textwidth]{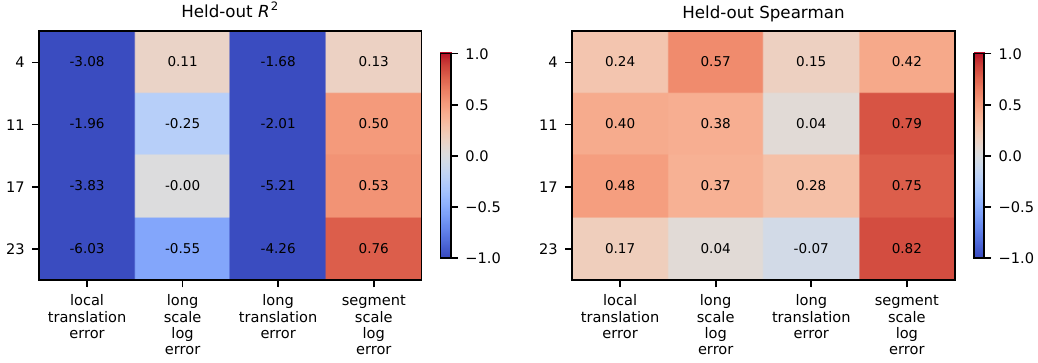}
    \caption{Linear probing of frozen Geometric Linear Attention states. Ridge regressors trained on chunk-level state features predict four geometric error targets. Segment-scale log error is the most reliably predictable, suggesting that the recurrent state encodes measurable metric information. Other targets are less consistently captured by a linear probe.}
    \label{fig:probe_performance}
  \end{minipage}
  \hfill
  \begin{minipage}[t]{0.41\textwidth}
    \centering
    \includegraphics[width=\textwidth]{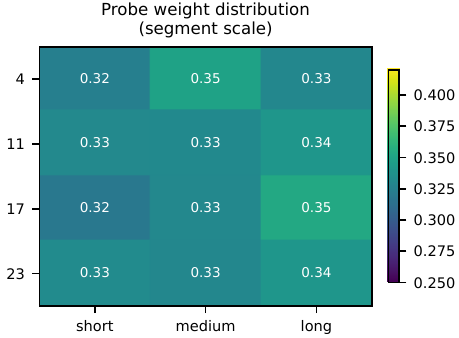}
    \caption{Probe weight attribution by retention band (segment scale target). Weights are nearly uniform across short / medium / long, showing that scale signal is distributed rather than band-specific.}
    \label{fig:probe_bands}
  \end{minipage}
\end{figure}

\clearpage

\end{document}